\documentclass{article}

\PassOptionsToPackage{numbers, compress}{natbib}

\usepackage[preprint]{neurips_2025}

\usepackage[utf8]{inputenc} 
\usepackage[T1]{fontenc}    
\usepackage{hyperref}       
\usepackage{url}            
\usepackage{booktabs}       
\usepackage{amsfonts}       
\usepackage{nicefrac}       
\usepackage{microtype}      
\usepackage{xcolor}         

\title{Geometric Kolmogorov-Arnold Superposition Theorem}

\author{%
Francesco Alesiani$^{*}$ \quad
Takashi Maruyama$^{*}$ \quad
Henrik Christiansen \quad
Viktor Zaverkin \\
NEC Laboratories Europe, \\ Heidelberg, Germany\\
}

\usepackage{microtype}
\usepackage{graphicx}
\usepackage{subfigure}
\usepackage{booktabs} 

\usepackage{hyperref}

\usepackage{amsmath}
\usepackage{amssymb}
\usepackage{mathtools}
\usepackage{amsthm}

\usepackage[capitalize,noabbrev]{cleveref}

\theoremstyle{plain}
\newtheorem{theorem}{Theorem}[section]

\newtheorem{lemma}[theorem]{Lemma}
\newtheorem{corollary}[theorem]{Corollary}

\theoremstyle{definition}

\theoremstyle{plain}
\newtheorem{remark}[theorem]{Remark}

\usepackage[framemethod=tikz]{mdframed}
\usetikzlibrary{shadows}
\definecolor{captiongray}{HTML}{555555}
\mdfsetup{%
  innertopmargin=1.0ex,
  innerbottommargin=0.0ex,
  innerleftmargin=0.5em,
  innerrightmargin=0.5em,
  linecolor=captiongray,
  linewidth=0.5pt,
  roundcorner=1pt,
  shadow=false
}
\usepackage{bm}
\usepackage{makecell}
\usepackage{tabu, booktabs}
\usepackage{hyperref}
\usepackage{url}
\usepackage{graphicx}
\usepackage{makecell}
\usepackage{float}
\usepackage[skip=3pt]{caption}
\newcolumntype{"}{@{\hskip\tabcolsep\vrule width 1pt\hskip\tabcolsep}}

\usepackage{amsmath,amsfonts,bm}

\def\eqref#1{equation~\ref{#1}}

\def\1{\bm{1}}

\DeclareMathAlphabet{\mathsfit}{\encodingdefault}{\sfdefault}{m}{sl}
\SetMathAlphabet{\mathsfit}{bold}{\encodingdefault}{\sfdefault}{bx}{n}

\newcommand{\R}{\mathbb{R}}

\newcommand{\NN}{\mathbb{N}}

\newcommand{\RR}{\mathbb{R}}

\newcommand{\bb}{\boldsymbol{b}}

\newcommand{\bA}{\boldsymbol{A}}
\newcommand{\bx}{\boldsymbol{x}}

\newcommand{\bX}{\boldsymbol{X}}

\usepackage{booktabs}
\usepackage{tabularx,ragged2e,siunitx}
\newcolumntype{Y}[1]{>{\Centering\hspace{0pt}\hsize=#1\hsize}X}

\usepackage{caption}
\usepackage{subcaption}
\usepackage{multirow}

\usepackage{comment}

\usepackage[inline]{enumitem}

\usepackage{xcolor}

\usepackage{listings}

\definecolor{codegreen}{rgb}{0,0.6,0}
\definecolor{codegray}{rgb}{0.5,0.5,0.5}
\definecolor{codepurple}{rgb}{0.58,0,0.82}
\definecolor{backcolour}{rgb}{0.95,0.95,0.92}

\lstdefinestyle{mystyle}{
    backgroundcolor=\color{backcolour},   
    commentstyle=\color{codegreen},
    keywordstyle=\color{magenta},
    numberstyle=\tiny\color{codegray},
    stringstyle=\color{codepurple},
    basicstyle=\ttfamily\footnotesize,
    breakatwhitespace=false,         
    breaklines=true,                 
    captionpos=b,                    
    keepspaces=true,                 
    numbers=left,                    
    numbersep=5pt,                  
    showspaces=false,                
    showstringspaces=false,
    showtabs=false,                  
    tabsize=2
}

\usepackage{wrapfig}

\usepackage{tcolorbox}
\tcbuselibrary{theorems}

\begin{document}

\maketitle

\newcommand{\tm}[1]{{{\textcolor{red}{[TM: #1]}}}}
\newcommand{\fa}[1]{{{\textcolor{blue}{[FA: #1]}}}}

\newcommand{\method}{\textsc{GKSN}}

\begin{abstract}
The Kolmogorov-Arnold Theorem (KAT), or more generally, the Kolmogorov Superposition Theorem (KST), establishes that any non-linear multivariate function can be exactly represented as a finite superposition of non-linear univariate functions. Unlike the universal approximation theorem, which provides only an approximate representation without guaranteeing a fixed network size, KST offers a theoretically exact decomposition. The Kolmogorov-Arnold Network (KAN) was introduced as a trainable model to implement KAT, and recent advancements have adapted KAN using concepts from modern neural networks. However, KAN struggles to effectively model physical systems that require inherent equivariance or invariance geometric symmetries as 
$E(3)$ 
transformations, a key property for many scientific and engineering applications. In this work, we propose a novel extension of KAT and KAN to incorporate equivariance and invariance over 
various 
group actions,
including $O(n)$, $O(1,n)$, $S_n$ and general $GL$, 
enabling accurate and efficient modeling of these systems. Our approach provides a unified approach that bridges the gap between mathematical theory and practical architectures for physical systems, expanding the applicability of KAN to a broader class of problems. We provide experimental validation on molecular dynamical systems and particle physics. 
\end{abstract}

\section{Introduction}
Kolmogorov Arnold Networks (KANs) \cite{Liu_Wang_Vaidya_2024} 
have recently risen to the interest of the machine learning community as an alternative to the well-consolidated Multi-Layer Perceptrons (MLPs) \cite{hornik1989multilayer}. MLPs have transformed machine learning for their ability to approximate arbitrary functions and have demonstrated their expressive power, theoretically guaranteed by the universal approximation theorem \cite{hornik1989multilayer}, in countless applications.  
The Kolmogorov-Arnold Theorem (KAT), developed to solve Hilbert’s 13th problem \cite{kolmogorov1961representation}, is a powerful and foundational mathematical result.
While the universal approximation theorem states that any function can be approximated with an MLP function of bounded dimension, KAT provides a way to exactly and with a finite and known number of univariate functions to represent any multivariate function. KAT has found multiple applications in mathematics \cite{laczkovich2021superposition}, fuzzy logic  \cite{kreinovichNORMALFORMSFUZZY1996}, pattern recognition \cite{koppen2002training}, and neural networks \cite{kuurkova1992kolmogorov,liu2024kan}. 

We have recently witnessed the flourishing of extensions of the use of KAN as a substitute for MLP \cite{Ji_Hou_Zhang_2024}, either as a plug-in replacement of MLP \cite{xu2024kaneffectiveidentifyingtracking, decarlo2024kolmogorovarnoldgraphneuralnetworks}, 
as a surrogate function for solving or approximating partial differentiable equations (PDE)
\cite{abueidda2024deepokandeepoperatornetwork,wang2024kolmogorovarnoldinformedneural,shuai2024physicsinformedkolmogorovarnoldnetworkspower}. Further KANs have been extended by exploring alternative basis such as Chebychev polynomials
\cite{ss2024chebyshevpolynomialbasedkolmogorovarnoldnetworks,mostajeran2024epickanselastoplasticityinformedkolmogorovarnold}, wavelet functions \cite{bozorgasl2024wavkanwaveletkolmogorovarnoldnetworks}, Fourier series \cite{xu2024fourierkangcffourierkolmogorovarnoldnetwork}, or alternative representations \cite{Guilhoto_Perdikaris_2024}. 

In applications to scientific computing, key physical symmetries need to be modeled \cite{Finzi2021,goodman2009symmetry,noether}. For example, the invariance to translations, rotations, and reflections (i.e. $E(3)$ group) of energies, as relevant for interatomic potentials in chemistry
~\cite{Zaverkin_Alesiani_Maruyama_Errica_Christiansen_Takamoto_Weber_Niepert_2024}. 
Other notable examples include fluid dynamics, partial differentiable equations (PDEs), astrophysics, material science, and biology. 
While MLP-based architectures have been widely explored \cite{schutt2017schnet,Batatia_Kovács_Simm_Ortner_Csányi_2023,Satorras_Hoogeboom_Welling_2022,Liao_Smidt_2023,Zaverkin_Alesiani_Maruyama_Errica_Christiansen_Takamoto_Weber_Niepert_2024}, it is not clear how to model a physical system with KAN-based architectures, especially since KAN models have shown potential to overcome the curse of dimensionality \cite{lai2021kolmogorov,poggio2022deep}. 

\vspace{2mm}
Our contributions are : 
\begin{itemize*}
\item to extend KAN to include $O(n)$ symmetries, thus been able to represent $O(n)$ invariant and equivariant functions (\cref{sec:geo}). We further extend the results to include the permutation invariance with respect to input data, which reduces the parameter count of the network and improves generalization. 

\item After providing the theoretical justification, we present practical architectures (\cref{sec:arch}) and analyze their performances with scientifically inspired experiments. We analyze the learning capability of the proposed KAN model for an idealized model (\cref{sec:lj}), which allows us to simulate multiple particles in multiple dimensions. 

\item To further analyze the learning capability of the proposed model, we experiment on real datasets for material design, the MD17 (\cref{sec:md17}) and MD22 (\cref{sec:md22}); but also particle physics with Top-tagging (\cref{sec:top-tagging}) and Quark-gluon tagging (\cref{sec:quark-tagging}). 

\item Extensive formal theorems and proofs are provided in the supplementary material (\cref{annex:KAT}) to support our claims summarized in \cref{tab:kst}.
\end{itemize*}

\section{Related Works}
\paragraph{Symmetry preserving machine learning architecture} 
Machine learning interatomic potentials (MLIPs) 
have emerged as powerful tools for modeling interatomic interactions in molecular and materials systems, offering a computationally efficient alternative to traditional ab initio methods. Architectures like Schnet \cite{schutt2017schnet} use continuous-filter convolutional layers to capture local atomic environments and message passing, enabling accurate predictions of molecular properties. To further enhance physical expressivity, $E(3)$-equivariant architectures \cite{thomas2018tensor} have been developed, which respect the symmetries of Euclidean space (rotations, translations, and reflections) by design. 
These models
ensure that predictions of energies and forces are invariant, respectively equivariant, to group actions of E(3),
making them highly data-efficient for tasks like force field prediction in molecular dynamics. 
Equivariant or invariant architectures enhance data efficiency, accuracy, and physical consistency in tasks where input symmetries (e.g., rotation, reflection, translation) dictate output invariance or equivariance.
Symmetry-preserving architecture for the Lorentz group have been proposed architecture based on high-order tensor products as LoLa \cite{butter2018deep}, LBN \cite{erdmann2019lorentz} LGN \cite{bogatskiy2020lorentz}, and LorentzNet \cite{gong2022efficient}, which introduce Minkowski dot product attention. Finally, permutation preserving models have been proposed to model function over sets, as DeepSet and subsequent models \cite{zaheer2017deep,amir2023neural}.
The advantage of KAN architecture has not yet been explored, we thus take a fundamental step in this direction with our study. 
\vspace{-1mm}
\paragraph{KAN Architectures}
Kolmogorov-Arnold Networks (KANs) are inspired by the Kolmogorov-Arnold representation theorem, which provides a theoretical foundation for approximating multivariate functions using univariate functions and addition. Early work by Hecht-Nielsen (1987) \cite{hecht1987kolmogorov} introduced one of the first neural network architectures based on this theorem, demonstrating its potential for efficient function approximation. 
\cite{lai2021kolmogorov} study the approximation capability of KST-based models in high dimensions and how they could potentially break the curse of dimension \cite{poggio2022deep}. 
\cite{ferdausKANICEKolmogorovArnoldNetworks2024} propose to combine  Convolutional Neural Networks (CNNs) with Kolmogorov Arnold Network (KAN) principles.
Additionally, 
\cite{yangKolmogorovArnoldTransformer2024}
explored the integration of KAN principles into transformer models, achieving improvements in efficiency for sequence modeling tasks. 
\cite{huEKANEquivariantKolmogorovArnold2024a} propose EKAN, an approximation method for incorporating matrix group equivariance into KANs. While these studies highlight the versatility of KAN architectures in adapting to various neural network frameworks, the extension to physical and geometrical symmetries has not been fully considered.

\paragraph{Theoretical Work on KAN}
The theoretical foundations of Kolmogorov–Arnold Networks (KANs) are rooted in the Kolmogorov–Arnold representation theorem, established by Andrey Kolmogorov  \citet{kolmogorov1957representation} and later refined by Vladimir Arnold \citet{arnold1959functions}. 
Building upon this foundation, David Sprecher \citet{sprecher1965structure} and George Lorentz \citet{lorentz1976approximation} provided constructive algorithms to implement the theorem, enhancing its applicability in computational contexts. 
Recent theoretical advancements have addressed challenges in training KANs, such as non-smooth optimization landscapes. Researchers have proposed various techniques to improve the stability and convergence of KAN training, including regularization methods \cite{Braun2009constructive} like dropout and weight decay, as well as optimization strategies involving adaptive learning rates, while \cite{igelnik2003kolmogorov} have proposed using cubic spline as activation and internal function for efficient approximation.  
These contributions have been instrumental in bridging the gap between the mathematical foundations of KANs and their practical implementation in machine learning.
However, training with energies requires fitting highly non-linear functions. In this work, we demonstrate how extending the KAN architecture enhances the learning capacity of KAT-based models.

\section{Background}
\paragraph{Equivariance and invariance} We call a function $\phi: X \to Y$ {\it equivariant} or {\it invariant}, if given a set of transformation $T^X_g$ on $X$, the input space, for a given element $g$ of action group $G$, there exists an associated transformation $T^Y_g: Y \to Y$ on the output space $Y$, such that 
\begin{align}\label{eq:equivariance}
\underbrace{ \phi(T^X_g(\bm{x})) = T^Y_g(\phi(\bm{x}))}_{\text{equivariant}}, ~~~\text{or}~~~ \underbrace{ \phi(T^X_g(\bm{x})) = \phi(\bm{x})}_{\text{invariant}}.
\end{align}
An example of $\phi$ is a non-linear function of a multivariate variable $\bm{x}=(\bm{x}_1, \dots, \bm{x}_m) \in \mathbb{R}^{m \times n}$ representing a point cloud with $m$ points, where each point lives in an $n$-dimensional space $\bm{x}_i \in \mathbb{R}^n$, $\phi(\bm{x})=\bm{y} \in \mathbb{R}^{m \times n}$ the transformed points, with $T_g$ a translation of the input $T^X_g(\bm{y}) = \bm{x} + \bm{g}$ and $T^Y_g$ an associated translation in the output domain $T^Y_g(\bm{y}) = \bm{y} + \bm{g}$. When $\phi$ is equivariant with respect to the action of $G$, then first applying the translation in the input domain and then applying $\phi$, is equivalent to first applying $\phi$ and then translating for the same amount $g$, in the target domain. When $\phi$ is invariant with respect to $G$, then applying the translation or not, results in the same output $\phi(\bm{x}+\bm{g})=\phi(\bm{x})=\bm{y}$. In this work, we consider three types of symmetries, i.e. invariance and equivariance:
\begin{itemize}
    \item {\it translation symmetry}: $\phi(\bm{x}+\bm{g})=\phi(\bm{x})$ for the invariance and $\phi(\bm{x}+\bm{g})=\phi(\bm{x})+\bm{g}$ for equivariance, with $\bm{g} \in \mathbb{R}^n$ and where $\bm{x}+\bm{g}$ refers to the element-wise operation $(\bm{x}_1 +\bm{g}, \dots, \bm{x}_m + \bm{g})$;
    \item {\it rotation and reflection symmetry}: given an orthogonal matrix $\bm{Q} \in \mathbb{R}^{n \times n}$,  $\phi$ is invariant or equivariant if $\phi(\bm{Q}\bm{x})=\phi(\bm{x})$ or $\phi(\bm{Q}\bm{x})=\bm{Q}\phi(\bm{x})$, and where $\bm{Q}\bm{x}$ refers to the element-wise operation $(\bm{Q}\bm{x}_1 , \dots, \bm{Q}\bm{x}_m)$;  
    \item {\it permutation symmetry}: $\phi$ is invariant or equivariant, if $\phi(\bm{x}_1 , \dots, \bm{x}_m)=\phi(\bm{x}_{\pi_1} , \dots, \bm{x}_{\pi_m})$ and $\phi(\pi(\bm{x}))=\pi(\phi(\bm{x}))$, for any permutation $\pi: [m] \to [m]$, where $\pi(\bm{x})=\bm{x}_{\pi_1} , \dots, \bm{x}_{\pi_m}$ .
\end{itemize}

We extend KAT in \autoref{sec:geo} to functions that exhibit these symmetries.

\paragraph{First Fundamental Theorem (FFT) of $GL(V)$} 
According to the First Fundamental Theorem \cite{kraftCLASSICALINVARIANTTHEORYa}, the ring of invariant polynomial functions can be generated by the invariants of the symmetry group. \cite{villarScalarsAreUniversal2023} shows how the FFT can be used to represent, among others, $O(n)$ and $O(1,n)$ invariant functions and their use in MLP.  
The FFT states that the ring of invariants for the action of $\mathrm{GL}(V)$ on $V^p \oplus V^{*q}$ is generated by the invariants $(i \mid j)$:
$K[V^p \oplus V^{*q}]^{\mathrm{GL}(V)} = K[(i \mid j) \mid i = 1, \ldots, p,\, j = 1, \ldots, q].$ \cref{annex:KAT} contains additional information.

\paragraph{Kolmogorov superposition theorem (KST)}
The Kolmogorov-Arnold representation theorem (KAT), proposed by  \citet{kolmogorov1961representation}, provides a powerful theoretical tool to represent a multivariate function $f(x_1,\dots,x_m)$ as the composition of functions of a single variable. The original form of KAT states that a given  continuous function $f: [0,1]^m \to \mathbb{R}$ can be represented exactly as 
\begin{align}\label{eq:kst}
f(x_1,\dots,x_m) = \sum_{q=1}^{2m+1} \psi_q(\sum_{p=1}^{m} \phi_{qp}(x_p))
\end{align}
with $\psi_q: \mathbb{R} \to \mathbb{R}$ and $\phi_{qp}: [0,1] \to \mathbb{R}$ uni-variate continuous functions. 

\begin{table*}
\centering
\caption{
Kolmogorov superposition formulas \cite{Guilhoto_Perdikaris_2024} for a continuous function $f(x_1, \dots, x_d)$ or $f(\bm{x}_1, \dots, \bm{x}_m)$ and their complexity in terms of parameters. $\langle \bm{x}_i,\bm{y}_j \rangle$ is either Euclidean or Mikowski scalar product, while $( \bm{x}_i | \bm{y}_j )$ is the $GL(n)$ invariant as defined in \cref{annex:KAT}.
}

\label{tab:kst}
\centering
\resizebox{\linewidth}{!}{%
\begin{tabular}{cccccc}
\toprule
    \textbf{Version} & \textbf{Formula} & \textbf{\makecell{Inner \\ Functions}} & \textbf{\makecell{Outer \\ Functions}} & \textbf{\makecell{Other \\ Parameters}} & \textbf{\makecell{Symmetry \\ Group}}
    \\ \midrule
    
    \makecell{Kolmogorov \\ (1957)} & $\sum_{q=1}^{2m+1}\psi_q\left( \sum_{p=1}^m \phi_{q,p}(x_p) \right)$ & $(2m+1)m$ & $2m+1$  & N/A & - \\ \hline

    \makecell{Ostrand \\ (1965)} & $\sum_{q=1}^{2mn+1}\psi_q\left( \sum_{p=1}^d \phi_{q,p}(\bm{x}_p) \right)$ & $(2nm+1)m$ & $2mn+1$  & N/A & - \\ \hline

    \makecell{Lorentz \\ (1962)} & $\sum_{q=1}^{2m+1} \psi \left( \sum_{p=1}^m \lambda_p \phi_q(x_p) \right) $ & $2m+1$ & $1$  & $\lambda\in\mathbb{R}^{m} $ & -\\ \hline
    
    \makecell{Sprecher \\ (1965)} & $\sum_{q=1}^{2m+1} \psi_q\left( \sum_{p=1}^m \lambda_p \phi(x_p+qa) \right) $ & $1$ & $2m+1$  & $a\in\mathbb{R},\  \lambda\in\mathbb{R}^d$  & - \\ \hline
        
    \makecell{Kurkova \\ (1991)} & $\sum_{q=1}^{N} \psi \left( \sum_{p=1}^m w_{pq} \phi_q(x_p) \right) $ & $2m+1 \le N$ & $1$  & $w \in \mathbb{R}^{m\times N}$ & - \\ \hline

    \makecell{Laczkovich \\ (2021)} & $\sum_{q=1}^{N} \psi \left( \sum_{p=1}^d \lambda_{pq} \phi_q(x_p) \right) $ & $N$ & $1$  & $\lambda \in\mathbb{R}^{m\times N}$ & - \\ \hline

    \makecell{\textbf{This work} \\ \cref{th:on_v1}} & $
\sum_{q=1}^{2m^2+1} \psi_q \left(\sum_{i=1,j=1}^{m,m} \phi_{qij}(\langle \bm{x}_i,\bm{x}_j \rangle) \right)    
    $ &$ (2m^2+1) m^2$ & $2m^2+1$ & N/A & $O(n)$,$O(1,n)$\\ \hline

    \makecell{\textbf{This work} \\ \cref{th:on_v2}} & 
    \makecell{
    $ \sum_{q=1}^{2mn+1} \psi_q \left( \sum_{i=1,j=1}^{m,n} \phi_{qij}(\langle \bm{x}_i,\bm{y}_j \rangle)  \right. $ \\ 
    $ \left. +\sum_{i=1,j=1}^{n,n} \phi'_{qij} (\langle \bm{y}_i, \bm{y}_j\rangle)
\right)$ } &  \makecell{ $(2mn+1) \times $ \\ $(mn+n^2)$} & $2mn+1$ & N/A & $O(n)$,$O(1,n)$ \\ \hline

    \makecell{\textbf{This work} \\ \cref{th:on_v3}} & $
    \sum_{q=1}^{2mn+1} \psi_q \left( \sum_{i=1,j=1}^{m,n} \phi_{qij}(\langle \bm{x}_i,\bm{x}_j \rangle) \right)    
    $ &$(2mn+1)mn$ & $2mn+1$ & N/A & $O(n)$,$O(1,n)$ \\ \hline

    \makecell{\textbf{This work} \\ \cref{lm:prem-v2}} & $
    F \left( \sum_{i=1}^{m} \phi_{1}({x}_i), \dots, \sum_{i=1}^{m} \phi_{2m+1}({x}_i) \right)    
    $ &$M=(2m+1)$ & $(2M+1)M$ & N/A & $S_n$ \\ \hline

    \makecell{\textbf{This work} \\ \cref{th:gl_v1}} & $
\sum_{q=1}^{2m^2+1} \psi_q \left(\sum_{i=1,j=1}^{m,m} \phi_{qij}(  ( \bm{x}_i | \bm{x}_j ) ) \right)    
    $ &$ (2m^2+1) m^2$ & $2m^2+1$ & N/A &  
    \makecell{$GL(n)$ over \\  polynomial ring }
      \\     
    
    \bottomrule
\end{tabular}
}
\end{table*}

\paragraph{Ostrand superposition theorem (OST)
}
In 1965, \citet{ostrandDIMENSIONMETRICSPACESa} proposed an extension of the original KAT to input compact domains. The theorem states that, given $X^p$ compact metric spaces of finite dimension $d_p=|X^p|$, such that $\sum_{p=1}^m d_p=M$, a continuous function $f: \prod_{p=1}^m X^p \to \mathbb{R}$ is representable in the form 
\begin{align} \label{ostrand}
f(\bm{x}_1,\dots,\bm{x}_m) = \sum_{q=1}^{2M+1} \psi_q(\sum_{p=1}^{m} \phi_{qp}(\bm{x}_p))    
\end{align}
with $\bm{x}_p \in X^p$, and $\phi_{qp}: X^p \to \mathbb{R}$ continuous functions. When $d_p=n, \forall p$, then $M=nm$. The difference between KAT and OST, is that the building functions $\phi_{qp}$ in OST are not defined on scalars (not any more uni-variate), but defined over arbitrary compact spaces $X^p$ (thus multi-variate).

While the original formulation has been criticized \cite{Girosi1989}, other versions of the original superposition theorem have been proposed to counter-argument the smoothness and efficiency of the representation \cite{Kourkova1991}. \autoref{tab:kst} summarizes the various versions of the KAT 
\cite{kolmogorov1957representation,braun2009application,Kourkova1991,kuurkova1992kolmogorov,laczkovich2021superposition,sprecher1963dissertation,sprecher1996numerical}.

\section{Geometric Kolmogorov Superposition Theorem}
\label{sec:geo}
We want to extend the KST to invariant functions to action $g \in O(n)$ or $g \in O(1,n)$. 
While the original KST already tells us that we can represent the original function as the superposition of univariate functions \autoref{eq:kst}, which requires a total of $(mn+1)(2mn+1)$ univariate functions, we would like to have a better form of this representation. OST teaches us that we only need $(m+1)(2mn+1)$ functions to represent a multivariate function on $(\mathbb{R}^n)^m$ and these functions take values from $\mathbb{R}^n \to \mathbb{R}$, therefore they are not univariate.
However, we claim that we can represent a generic invariant function $f(\bm{x})$ using only univariate functions, as  
\begin{align} \label{eq:inv_v1}
f(\bm{x}_1,\dots,\bm{x}_m) = \sum_{q=1}^{2m^2+1} \psi_q \left(\sum_{i=1,j=1}^{m,m} \phi_{qij}(\langle \bm{x}_i,\bm{x}_j \rangle) \right),    
\end{align}
more formally stated and proved in \autoref{th:on_v1}, the results is intuitive given that 
$\langle \bm{x}_i,\bm{x}_j \rangle_{i,j=1}^{n}$, with $\langle \bm{x}_i,\bm{x}_j \rangle$ either the Euclidean or Mikowski scalar product,  
represent a complete set of invariant features \cite{villarScalarsAreUniversal2023}. Unfortunately, this form is $m^4$ in the number of nodes. In \autoref{th:on_v2}, we provided an improved version of the geometric KST that grows $m^2$ with the number of nodes, since it only uses a linear number of invariant features. Indeed, if we select $\bm{y}_j = \alpha_j (\bm{x}_1,\dots,\bm{x}_m)$ a linear combination of the inputs such that they span the full space $\mathbb{R}^n$:  
\begin{align} 
&f(\bm{x}_1,\dots, \bm{x}_m) = 
&\hspace{-1mm}\sum_{q=1}^{2mn+1} \psi_q \left( \sum_{\substack{1 \leq i \leq m, \\ 1 \leq j \leq n}} \phi_{qij}(\langle \bm{x}_i,\bm{y}_j \rangle) \right. \nonumber 
\left. +  \sum_{\substack{1 \leq i \leq n, \\ 1 \leq j \leq n}} \phi'_{qij} (\langle \bm{y}_i,\bm{y}_j\rangle)
\right),
\end{align}
in which $\langle \bm{x}_p,\bm{y}_j \rangle_{j=1}^{n}= \{ \langle \bm{x}_p,\bm{y}_1 \rangle \dots \langle \bm{x}_p,\bm{y}_n \rangle \}$. While the formal statement and proof are given in \autoref{th:on_v2}, the intuition is that we can project the input on the vectors $\bm{y}_j$. Since these vectors, built as linear combinations of the input, do not form an orthonormal basis, we need the information of their inner product $\langle \bm{y}_i,\bm{y}_j\rangle$ to reconstruct the invariant features $\langle \bm{x}_i,\bm{x}_j \rangle$. 
If we further restrict the vectors $\bm{y}_j$ to be a fixed subset of the input features, we have that \autoref{th:on_v3}, 
\begin{align} 
\label{eq:inv_v3}
f(\bm{x}_1,\dots,\bm{x}_m) = \sum_{q=1}^{2mn+1} \psi_q \left( \sum_{i=1,j=1}^{m,n} \phi_{qij}(\langle \bm{x}_i,\bm{x}_j \rangle)
\right),    
\end{align}
which further reduces the need for the additional $n^2$ invariant features. We formalize this last result, but a more extensive theoretical derivation is provided in \cref{sec:theory}:
\begin{tcolorbox}[title=$O(n)$ invariance - v3]
\begin{corollary}
\label{th:on_v3-main}
Suppose that  $\operatorname{span}(\{\bm{x}_j\}_{j=1}^n)=\mathbb{R}^n$. Then, a continuous function invariant to the action of $O(n)$ $f(\bm{x}_1,\dots,\bm{x}_m): X^m \to \mathbb{R}$, with $X \subset \mathbb{R}^n$ a compact space, can be represented as $f(\bm{x}_1,\dots,\bm{x}_m) = \sum_{q=1}^{2mn+1} \psi_q \left( \sum_{i=1,j=1}^{m,n} \phi_{qij}(\langle \bm{x}_i,\bm{x}_j \rangle)
\right).$
\end{corollary}
\end{tcolorbox}
\paragraph {Equivariant $O(n)$ or $O(1,n)$ functions} While in the supplementary material (\autoref{annex:eqivariance}), we discuss the equivariant version of these results, we can build equivariant functions, from invariant functions \cite{villarScalarsAreUniversal2023}, as
\begin{align*}
f(\bm{x}_1,\dots,\bm{x}_m) = \sum_{l=1}^m f_l (\bm{x}_1,\dots,\bm{x}_m) \bm{x}_l   
\end{align*}
with $f_l (\bm{x}_1,\dots,\bm{x}_m)$ invariant functions. Further, we can use the gradient of a geometric invariant function to build equivariant representations
\begin{align*}
f(\bm{x}_1,\dots,\bm{x}_m) = \sum_{l=1}^m \nabla_{\bm{x}_l} f_l (\bm{x}_1,\dots,\bm{x}_m)
\end{align*}
\paragraph{Translation and permutation symmetry} Translation symmetry is obtained by removing the mean of the coordinate from the input, while the permutation invariant \autoref{annex:on_perm_in} is obtained by imposing the univariate function to not depend on the node index. 

\begin{figure*}[t!]
    \centering    
    \includegraphics[width=0.8\linewidth]{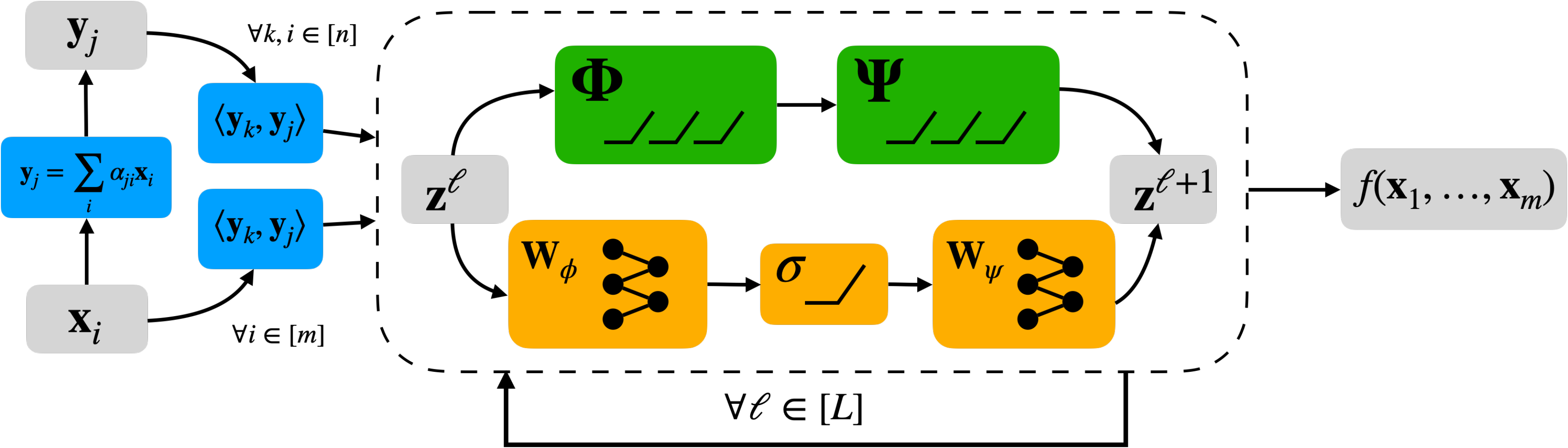} 
    \caption{The geometric Kolmogorov superposition network is composed of layers that comprise two terms. The first term is based on the classical KST function representation, while the second term, similar to a residual path, is an almost linear term that helps the training of the non-linear functions.}
    \label{fig:geo-kst}
\end{figure*}

\section{Geometric Kolmogorov Superposition Networks (\method{})} \label{sec:arch}
Finding the representation functions $\psi_q,\phi_{pq}$ is still a hard non-linear optimization problem. To reduce the training complexity, we consider a representation as a layer and allow the composition of multiple layers (\autoref{fig:geo-kst}). The fundamental result from \autoref{eq:inv_v3} is that we can use univariate functions on invariant features. We consider a single layer of the Geometric Kolmogorov Superposition Networks (\method{}) as the composition of the univariate functions $\phi^\ell_{pq}$ and the subsequent univariate functions $\psi^\ell_{q}$. With an abuse of notation and dropping $\ell$ dependence on the functions, we write 
\begin{align}\label{eq:geo-ksn}
\bm{z}_{\ell+1} = \overbrace{\bm{\Psi}}^{ l  \times k} \circ \overbrace{\bm{\Phi}^T}^{k \times m}(\bm{z}_{\ell}) + \overbrace{\bm{W}_\psi}^{l \times k'} \sigma ( \overbrace{\bm{W}_\phi^T}^{ k' \times m} \bm{z}_{\ell}),
\end{align}
or if we compute the $i$-th element, \\
\scalebox{0.95}{
\begin{minipage}{\linewidth}
\begin{align}
z^{\ell+1}_i = \underbrace{ \sum_{k} \psi_{ik} \left( \sum_j \phi_{jk} (z^{\ell}_j) \right) }_{\text{KST}}  
+  \underbrace{\sum_{k} \underbrace{w_{ik}^\psi \sigma }_{\psi_{ik}(.)} \left( \sum_{j} \underbrace{w_{ji}^\phi}_{\phi_{jk}(.)} z^{\ell}_j \right) }_{\text{Residue term}}, \nonumber
\end{align}
\end{minipage}
}
where $\circ$ is the function composition operator. 

The first term is the classical KST form, while the second is inspired by the newer forms (\autoref{tab:kst}), which contain linear terms, with a non-linear function $\sigma$ in the middle. We, therefore, assume that the original function can be represented as the sum of two functions, the first with smooth but non-linear univariate functions, the second with composition of a scaled non-linear function and the sum of linear functions. We further assume $\sigma$ to be a fix almost everywhere smooth, continuous and almost linear to improve the training of wide layers. The second path plays the role similar of the residual connection, which helps the training of the non-linear univariate functions. 

\section{Experimental Evaluation}
After presenting the experimental setup, we show the performance on representative datasets. To evaluate the representation power of the \method{} to model an invariant function to $E(3)$ symmetry action, we consider the task of training atomistic energy from atomic system configurations.  
We therefore considered the Lennard-Jones particle system (\cref{sec:lj}), Linear polymers (\cref{sec:linear-polymer}), the MD17 (\cref{sec:md17}), and MD22 (\cref{sec:md22}) datasets. 
We also experiment with the use of \method{} for Lorentz symmetries, in particular, we study 
on Top jets stream classification (\cref{sec:top-tagging}), Quark-Gluon tagging
(\cref{sec:quark-tagging}), and symmetry discovery
(\cref{sec:symmetry-discovery}).
\subsection{Experimental setup and baselines}
We compare different models to learn invariant functions from data, from both synthetic and real datasets. In the test, we normalize the output to the interval $[0,1]$.
\paragraph{Symmetries} We name $O(n)$ the models with rotation and reflection symmetry, while we use $\pi$ for the models that implement permutation symmetry. 
\paragraph{Networks}  We mainly compare against the use of two layers {\bf MLP} models.  
We implemented the {\bf KAN}  model of \autoref{eq:geo-ksn}, where we use ReLU \cite{glorot2011deep} both as the basis for the KAN non-linear functions ($\psi_q,\phi_{pq}$) and for the residual connection ($\sigma$). The name of the model contains two symbols $T$=True and $F$=False; the first boolean tells us if the node index is used as an additional $O(n)$ invariant feature. The effect of adding the index of the node is to emulate the non-permutation invariant function. The second boolean is used to show if the linear ($T$) (\autoref{eq:inv_v3}) or quadratic ($F$) (\autoref{eq:inv_v1}) feature is used. Therefore, $\pi ~ O(n)$ KAN($T,T$) is a permutation invariant model based on the KAN architecture, where node index is used as a feature, where the number of features is linear in the number of nodes $m$.  
\paragraph{Invariant Features} While \autoref{eq:inv_v3} tells us that we can represent any invariant function with the inner products, nevertheless, to improve expressivity, 
we extend the invariant feature to include: 
\begin{align}
\|\bm{x}_i \|,\|\bm{y}_j\|, \|\bm{x}_i-\bm{y}_j\|, \langle \bm{x}_i,\bm{y}_j \rangle, \sqrt{\|\bm{x}_i \|^2 \|\bm{y}_j\|^2-\langle \bm{x}_i,\bm{y}_j \rangle^2} \nonumber
\end{align}
As additional invariant features, we optionally include the node index (first flag), and when present (experiments with MD17 and MD22), we also include the atom type. We have not explored alternative ways to embed the node's additional information as input to the network. The last term is also equivalent to $\| x \otimes y\|$ in $n=3$ dimensions, with $\otimes$ the cross product. In \cref{sec:ablation}, we propose an ablation study on the effect of features on the representation power. 
\paragraph{Quadratic versus Linear features} 
A consequence of \autoref{eq:inv_v3}, with the associated theorem, is that the number of invariant features that we need is linear with the number of nodes. 
We nevertheless compare also with the quadratic version as in \autoref{eq:inv_v1}.

\begin{table}
\caption{Huber NLL ($\uparrow$, higher is better) for the LJ dataset
on different dimensions ($n \in [3,5]$) and different number of nodes $m \in [4,10,15]$. Standard deviation as superstript, mean computed over $3$ runs.}
\label{tab:lj1}
\begin{tabularx}{\columnwidth}{
@{} Y{1.5} Y{0.9} Y{0.9} Y{0.9} Y{0.9} @{}
}
\toprule
LJ  $m/n$ & $O(n)$ KAN & $O(n)$ MLP & $\pi~O(n)$ KAN  & $\pi~O(n)$ MLP \\
\midrule
4/3 & \textbf{8.41}$^{\pm 0.19}$ & 8.00$^{\pm 0.12}$ & 7.88$^{\pm 0.15}$ & 7.59$^{\pm 0.14}$ \\
10/3 & \textbf{7.10}$^{\pm 0.16}$ & 6.76$^{\pm 0.09}$ & \textbf{7.08}$^{\pm 0.28}$ & 5.33$^{\pm 0.18}$ \\
10/5 & \textbf{7.15}$^{\pm 0.37}$ & \textbf{6.71}$^{\pm 0.28}$ & \textbf{7.23}$^{\pm 0.41}$ & 3.72$^{\pm 0.60}$ \\
15/3 & \textbf{7.25$^{\pm 1.25}$} & \textbf{7.09}$^{\pm 1.10}$ & \textbf{7.28}$^{\pm 1.17}$ & 3.92$^{\pm 0.41}$ \\
15/5 & 6.73$^{\pm 0.18}$ & 6.56$^{\pm 0.13}$ & \textbf{6.96}$^{\pm 0.24}$ & 1.76$^{\pm 1.33}$ \\
\bottomrule
\end{tabularx}
\end{table}
\begin{figure*}
    \centering
    \subfigure[]{
        \includegraphics[width=0.4\linewidth]{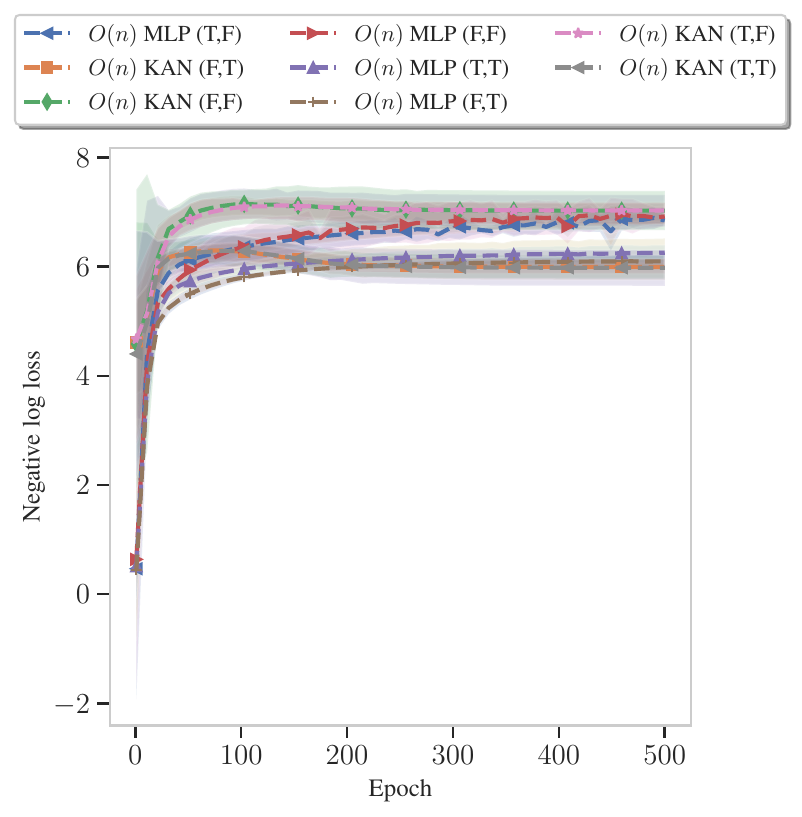}
    }%
    \subfigure[]{
        \includegraphics[width=0.43\linewidth]{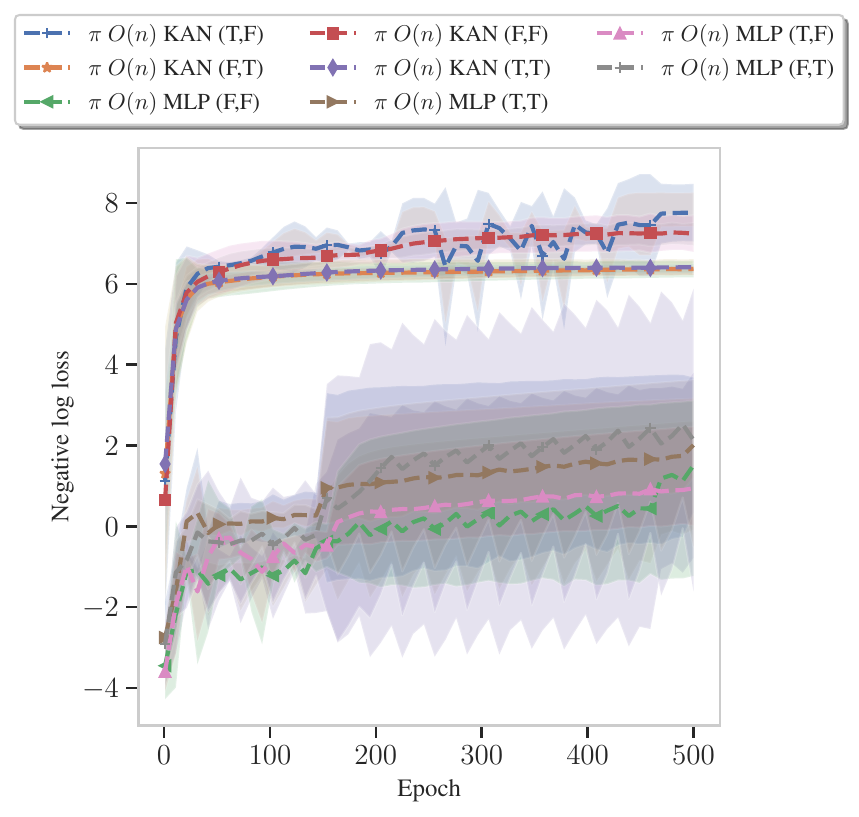}  
    }%

    \caption{
    a) Test performance (Negative log Huber Loss) of $O(n)$ invariant models for the LJ experiment with $n=5$ and $m=15$.  
    In parenthesis, the two flags indicate if the model includes the node index $(T,*)$ or not $(F,*)$; the second flag signals if the features are linear $(*,T)$ (according to \autoref{eq:inv_v3}) or quadratic $(*,F)$ (according to \autoref{eq:inv_v1}) in the number of nodes.
    b) Test performance (Negative log Huber Loss $\uparrow$) of $O(n)$ and permutation invariant models for the LJ experiment.
    }
    \label{fig:lj}
\end{figure*}

\subsection{Lennard-Jones experiments}
\label{sec:lj}
Lennard-Jones potential approximates inter-molecular pair interaction and models repulsive and attractive interactions. It captures key physical principles and it is widely used to model solid, fluid, and gas states. More details are in \autoref{annex:LJ}. 
\autoref{fig:lj} show the test regression loss during training for a system in $3$ dimensions and with $15$ nodes. The loss is plotted on a negative log scale. We use the Huber loss that is quadratic if the error is less than $1$, and linear if larger. The test loss for the $O(n)$ invariant model (\autoref{fig:lj}.(a)) is regular during training and all models seem to have similar results, while in \autoref{fig:lj}.(b) the performance of permutation invariant models have quite different behavior. The MLP-based models are more unstable, while KAN-based models have a much more regular performance. 
\autoref{tab:lj1} summarizes the regression accuracy at test time for all the models. The permutation invariance reduces the performances, but more remarkably on smaller systems.

\begin{table}
\caption{Huber NLL $\uparrow$ for the MD17 dataset (mean and standard deviation in parenthesis)}
\label{tab:md17}
\begin{tabularx}{\columnwidth}{
@{} Y{1.5} Y{0.9} Y{0.9} Y{0.9} Y{0.9} @{}
}
\toprule
Dataset (MD17) & $O(n)$ KAN & $O(n)$ MLP & $\pi~O(n)$ KAN & $\pi~O(n)$ MLP \\
\midrule
Aspirin & ${\bf 6.44}^{\pm 0.10}$ & $5.62^{\pm 0.01}$ & $5.69^{\pm 0.02}$ & $4.73^{\pm 0.27}$ \\
Benzene & ${\bf 7.66}^{\pm 0.08}$ & $5.93^{\pm 0.01}$ & $6.51^{\pm 0.17}$ & $5.64^{\pm 0.13}$ \\
Ethanol & ${\bf 7.57}^{\pm 0.04}$ & $5.44^{\pm 0.01}$ & $6.09^{\pm 0.13}$ & $5.49^{\pm 0.03} $\\
Malonaldehyde & ${\bf 7.50}^{\pm 0.05}$ & $5.39^{\pm 0.01}$ & $5.85^{\pm 0.04}$ & $5.38^{\pm 0.04}$\\
Naphthalene & ${\bf 6.85}^{\pm 0.07}$ & $5.35^{\pm 0.00}$ & $5.72^{\pm 0.09}$ & $4.65^{\pm 0.76}$ \\
Salicylic & ${\bf 6.96}^{\pm 0.09}$ & $5.62^{\pm 0.00} $& $5.83^{\pm 0.10}$ & $5.17^{\pm 0.24}$ \\
Toluene & ${\bf 7.05}^{\pm 0.13}$ & $5.68^{\pm 0.02}$ & $6.03^{\pm 0.10}$ & $5.40^{\pm 0.11}$ \\
Uracil & ${\bf 7.54}^{\pm 0.08} $& $5.65^{\pm 0.01} $ & $6.10^{\pm 0.11}$ &$ 5.52^{\pm 0.05}$ \\
\bottomrule
\end{tabularx}
\end{table}

\subsection{MD17}
\label{sec:md17}
MD17 dataset contains samples from a long molecular dynamics trajectory of a few small organic molecules \cite{chmiela2017machine}. For each molecule, we split into $8,000$ training and $200$ test configurations. 
In \autoref{tab:md17} we show the negative log of the Huber loss (negative log loss - NLL), aggregated over various model options, while in \cref{tab:md17_detailed} 
we provide the test loss for each model. 
\cref{fig:md22_md17}.(a-b)
show the Huber NLL at test time for the Toluene molecule for the two classes of models. 
The test loss in negative log scale at training for $O(n)$ invariant models in \autoref{fig:md22_md17}.(a)is stable, but reducing the number of features leads to lower performance, while KAN shows better accuracy. The training for the permutation invariant models in \autoref{fig:md22_md17}.(b) is less stable, and the overall performance reduces while keeping the model size smaller. 
\autoref{tab:md17} summarizes the performance of all models in the various atomic systems of MD17, the KAN-based models show consistently better performance, even with a smaller network size.

\begin{table}
\caption{Performance aggregated at the level of the model type for the MD22 dataset; the performance is the negative log of the Huber loss $\uparrow$ (mean and standard deviation in parenthesis); }
\label{tab:md22}
\begin{tabularx}{\columnwidth}{
@{} Y{1.5} Y{0.9} Y{0.9} Y{0.9} Y{0.9} @{}
}
\toprule
Dataset (MD22) & $O(n)$ KAN & $O(n)$ MLP & $\pi~O(n)$ KAN & $\pi~O(n)$ MLP \\
\midrule
AT-AT-CG-CG & ${\bf 8.02}^{\pm 0.14}$ & $7.61^{\pm 0.05} $&$ 7.73^{\pm 0.05}$ & $0.82^{\pm 0.32} $\\
AT-AT & ${\bf 7.32}^{\pm 0.21}$ & $6.56^{\pm 0.01}$ & $6.62^{\pm 0.03}$ & $0.82^{\pm 0.40}$ \\
Ac-Ala3-NHMe & ${\bf 5.77}^{\pm 0.07} $&$ 5.57^{\pm 0.00} $&$ 5.57^{\pm 0.01}$ &$ 1.48^{\pm 1.08} $\\
DHA & ${\bf 5.64}^{\pm 0.07} $&$ 5.52^{\pm 0.00} $&$ 5.50^{\pm 0.01} $&$ 0.04^{\pm 0.82} $\\
Buckyball-catcher & ${\bf 8.85}^{\pm 0.24} $&$ 7.27^{\pm 0.01} $&$ 7.41^{\pm 0.07} $&$ 0.21^{\pm 0.71} $\\
Stachyose & ${\bf 6.30}^{\pm 0.12} $&$ 5.70^{\pm 0.01} $&$ 5.73^{\pm 0.03} $&$ 1.36^{\pm 1.42} $\\
\bottomrule
\end{tabularx}
\end{table}

\subsection{MD22}
\label{sec:md22}
MD22 dataset \cite{chmiela2023accurate} contains samples from molecular dynamics trajectories of four major classes of biomolecules, as proteins, lipids, carbohydrates, nucleic acids, and supramolecules. In MD22, number of atoms ranges from $42$ to $370$. 
For each molecule, we split into $8,000$ training and $200$ test configurations. In \autoref{tab:md22} we show the NLL aggregated over various model options, while in \autoref{tab:md22_detailed} for more details information on the performance. 
\cref{fig:md22_md17}.(c-d)
show the Huber NLL at test time 
for the Ac-Ala3-NHMe  molecule, with and without permutation invariance. 

Similar to the MD17 dataset,  the test loss in negative log scale at training for the $O(n)$ invariant models reported in 
\cref{fig:md22_md17}.(c)
is stable for the KAN-based models, while MLP-based models show more unstable training and lower performance.  The training for the permutation invariant models in 
\cref{fig:md22_md17}.(d)
is even less stable for the MLPs, leading to low accuracy. 
\autoref{tab:md22} summarizes the performance of all models in the various atomic systems of MD22, the KAN-based models show consistently better performance, even with a smaller network size. 

\begin{table}
\caption{Top-tagging experimental results, including LGN \cite{bogatskiy2020lorentz}, 
LorentzNet \cite{gong2022efficient}, and other baselines 
\cite{komiske2019energy},
\cite{qu2020jet}, Results for EMLP-${\displaystyle SO(1,3)^+}$  and EKAN -${\displaystyle SO(1,3)^+}$ are from \cite{huEKANEquivariantKolmogorovArnold2024a}; 
$^*$ Train on $ 10^4$ samples 
}
\label{tab:top-tagging}
\begin{tabularx}{\columnwidth}{
X X X X X
}
\toprule
Model & Accuracy & AUC  & $1/\epsilon_B(0.5)$ & $1/\epsilon_B (0.3)$ \\
\midrule
ResNeXt      & $0.936$  & $0.9837$ & $302^{\pm 5}$  & $1147 \pm 58$    \\
P-CNN        & $0.930$  & $0.9803$ & $201^{\pm 4}$   & $759 ^{\pm 24}$    \\
PFN          & $0.932$      & $0.9819$ & $247^{ \pm 3}$    & $888^{\pm 17}$      \\
ParticleNet  & ${\bf0.940}$  & ${\bf0.9858}$ & ${\bf397^{\pm 7}}$   & ${\bf1615^{\pm 93}}$       \\
EGNN        & $0.922$  & $0.9760$ & $148^{\pm 8}$    & $540^{\pm 49}$    \\
LGN         & $0.929$  & $0.9640$ & $124^ {\pm 20}$  & $435^{\pm 95}$    \\ 
EMLP          & 0.771*  & -      &-&-  \\
EKAN       & 0.769*  & -      &-&-  \\ 
LorentzNet  & $\bm{0.942}$  & $\bm{0.9868}$ & $\bm{498 ^{\pm 18}}$   & $\bm{2195 ^{\pm 173}}$   \\ 
\midrule
\method{}       & ${\bf 0.940} $    & ${\bf0.9858}$& ${\bf445}^{\pm 28} $& ${\bf1634}^{\pm 328} $\\
\bottomrule
\end{tabularx}
\end{table}

\begin{figure*}
    \centering

    \subfigure[]{
        \includegraphics[width=0.4\linewidth]{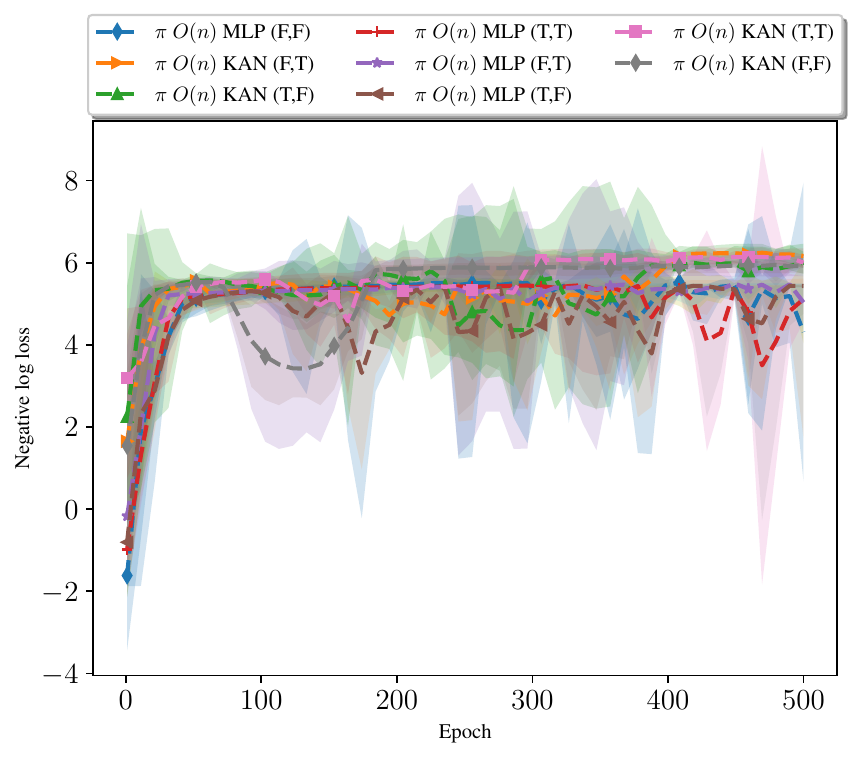} 
    }%
    \subfigure[]{
        \includegraphics[width=0.4\linewidth]{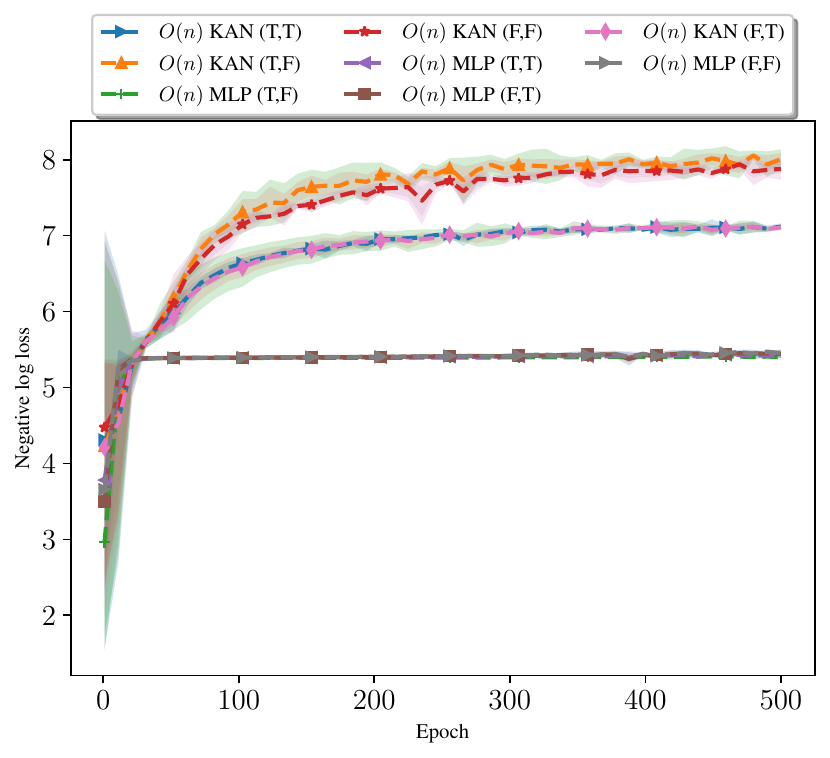} 
    }%
    \\
    \subfigure[]{
        \includegraphics[width=0.4\linewidth]{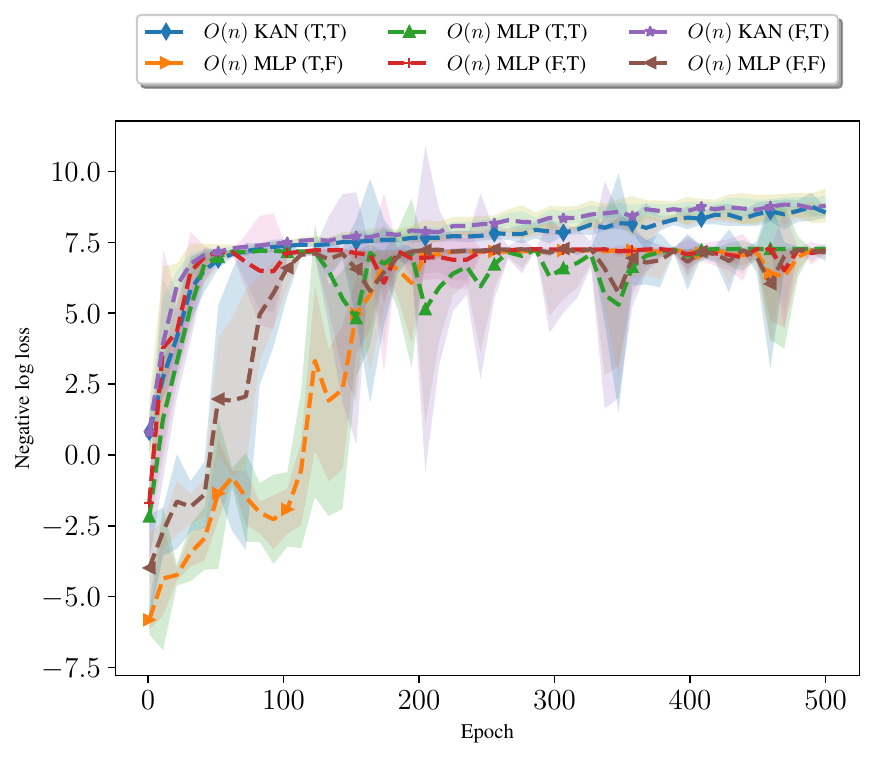}
    }%
    \subfigure[]{
        \includegraphics[width=.4\linewidth]{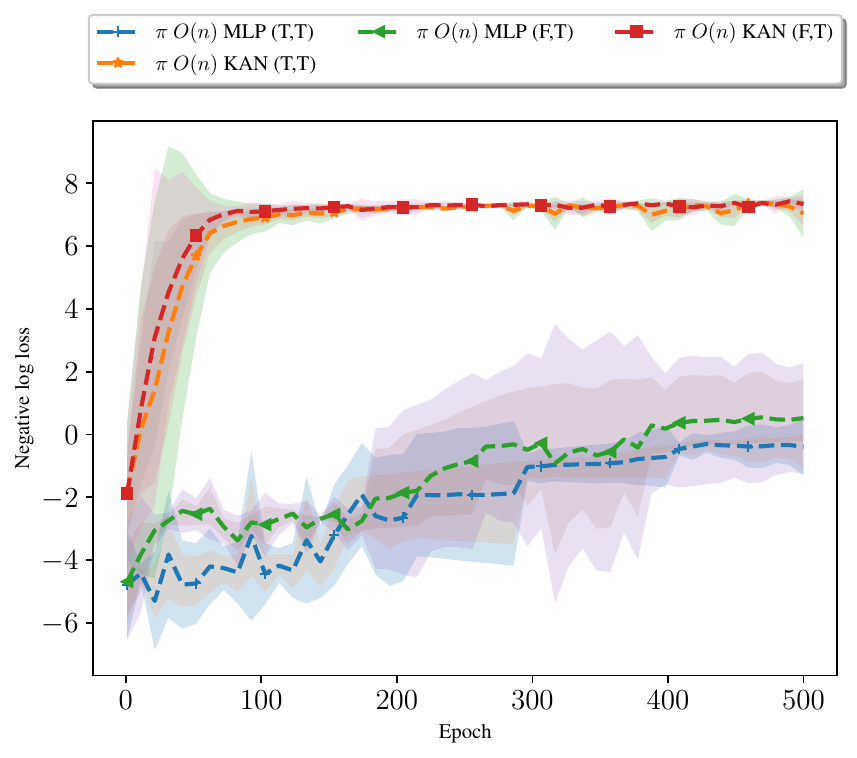}
    }%

    \caption{
    a) Test performance (Negative log Huber Loss $\uparrow$) of various models for the Ethanol dataset of MD17. $\pi~O(n)$ are the models that are invariant to rotation, reflection, and permutation.    
    b) Test performance (Negative log Huber Loss $\uparrow$) of various models for the Ethanol dataset of MD17. $O(n)$ is the model that is invariant to rotation and reflection on $\mathbb{R}^n$.    
    c) Test performance (Negative log Huber Loss $\uparrow$) of $O(n)$ invariant models for the Buckyball-Catcher dataset of MD22.
    d) Training performance (Negative log Huber Loss $\uparrow$) of  $O(n)$ and permutation invariant models for the Buckyball-Catcher dataset of MD22.
    }
    \label{fig:md22_md17}
\end{figure*}

\subsection{Top Tagging}
\label{sec:top-tagging}
Lorentz group $SO(1,3)^+$ is an important set of transformations in many physics problems. Top tagging dataset is an open benchmark dataset \cite{kasieczka2019machine} with the task of classifying between top quark jets and background jets.  
It consists of $2$M observations, each consisting four-dimensional momentum of up to $200$ particle jets. The classification task is Lorentz invariant, where the rotated or boosted input momentum belongs to the same category.

\begin{table}
\caption{Quark-gluon tagging experimental results. The LorentzNet, EGNN and LGN results are averaged over $6$ runs, \method{} over $3$. }
\label{tab:Quark-gluon-tagging}
\begin{tabularx}{\columnwidth}{
X X X X X
}
\toprule
Model     & Accuracy & AUC  & $1/\epsilon_B(0.5)$ & $1/\epsilon_B (0.3)$    \\
\midrule
ResNeXt   & $0.821$  & $0.8960$ & $30.9$  & $80.8$  \\
P-CNN    & $0.827$  & $0.9002$ & $34.7$  & $91.0$   \\
PFN    & - & $0.9005$ & $34.7^{\pm 0.4}$   & - \\
ParticleNet & ${\bf 0.840}$  & ${\bf 0.9116}$ & $39.8^{\pm 0.2}$  & $98.6^{\pm 1.3}$    \\
EGNN & $0.803$  & $0.8806$ & $26.3^{\pm0.3}$  & $76.6^{\pm 0.5}$       \\
LGN    & $0.803$  & $0.8324$ & $16.0$    & $44.3$  \\
LorentzNet  & ${\bf 0.844}$  & ${\bf 0.9156}$ & $42.4^{\pm0.4}$  & $110.2^{ \pm 1.3}$  \\
\midrule
\method{} & ${\bf 0.839}$  & ${\bf 0.9127}$ & $39^{\pm 0.4}$   & $101^{\pm3.5}$   \\
\bottomrule
\end{tabularx}
\end{table}

\subsection{Quark-gluon tagging}
\label{sec:quark-tagging}
In the Quark-gluon tagging dataset \cite{komiske2019energy}, the task consists of discriminating light-quark from gluon-initiated jets. The dataset consists of $2$  million jets in total, where half are gluon jets and half are background jets. The Quark-gluon tagging classification task is modelled with a Lorentz invariant function. \cite{bogatskiy2020lorentz}

\section{Conclusions}

We propose an extension of the KAN architecture for invariant and equivariant function representation, which is based on  
the theoretical results that provide us with a lower bound on the number of functions needed for approximating invariant functions.
The theoretical results in \autoref{sec:geo}, provide a considerable improvement with previous results \cite{villarScalarsAreUniversal2023}, reducing the complexity from quadratic to linear.
We further tested the performance and compared it with MLP-based architectures on an ideal physical system, the Lennard-Jones experiment, and on two real molecular datasets, the MD17 and the MD22 datasets. The performance of the proposed network architecture shows in our experiments improved performance with respect to MLP, and further investigation will show if this architecture can be extended to implement KAN-based machine learning interatomic potentials. 


\bibliographystyle{plainnat}
\bibliography{kolmogorov,r1references,mlip,scalar}


\newpage
\appendix
\onecolumn

\section*{Supplementary Material of Geometric Kolmogorov-Arnold Superposition Theorem}

\section{Main theorems for the Kolmogorov 
Superposition Theorem for invariant and equivariant functions }
\label{annex:KAT}
\label{sec:theory}

We first recall the Kolmogorov - Arnold
and Ostrand theorems.
\begin{theorem}{\cite{kolmogorov1961representation}} \label{th:kan}
For any integer $m \geq 2$ there are continuous real functions $\phi_{qp}(x)$ on the close unit interval $E = [0, 1]$ such that each continuous real function $f(x_{1}, \cdots, x_{m})$ on the $m$-dimensional unit cube $E^{m}$ is representable as 
\[f(x_1,\dots,x_m) = \sum_{q=1}^{2m+1} \psi_q(\sum_{p=1}^{m} \phi_{qp}(x_p)),\]
where $\psi_q$ are continuous functions.

\end{theorem}

\begin{theorem}{\cite{ostrandDIMENSIONMETRICSPACESa}}
\label{thm:ostrand}
For $p = 1, 2, \cdots, m$, let $X^p$ be a compat metric space of finite dimension $d_p$, and let $n = \sum_{p=1}^{m} d_{p}$. There exist continuous functions $\phi_{qp}: X^{p} \rightarrow [0, 1]$, for $p = 1, \cdots, m$ and $q = 1, 2, \cdots, 2n + 1$, such that every continuous real function $f$ defined on $\prod_{p=1}^{m} X^{p}$ is representable in the form 
\[f(x_1,\dots,x_m) = \sum_{q=1}^{2n+1} \psi_q(\sum_{p=1}^{m} \phi_{qp}(x_p)),\]
where the functions $\psi_q$ are real and continuous.
\end{theorem}

We also summarize the invariance representation from "Lemma 1",
and "Proposition 8" of \cite{villarScalarsAreUniversal2023}. 
\begin{theorem}{\cite{villarScalarsAreUniversal2023}} (First Fundamental Theorem of $O(n)$ and $O(1,n-1)$)
\label{thm:villar}
Suppose a function $f(x_1,\dots,x_m): (\mathbb{R}^{n})^m \to \mathbb{R}$ is an $O(n)$ or $O(1,n-1)$ continuous invariant scalar function. Then, $f$ can be represented as a continuous function of only scalar product of the input $x_i$. That is, there is a continuous function $g$ such that 
$f(x_1,\dots,x_m) = g(X^TX) = g((\langle x_i,x_j \rangle)_{i,j=1}^m)$, with $\langle x_i,x_j \rangle = x^T \Lambda x$ the invariant inner scalar product with metrics $\Lambda=1$ for $O(n)$ and $\Lambda=\text{diag}(1,-1,\dots,-1)$  for $O(1,n-1)$. 
\end{theorem}

The propositions in \cite{villarScalarsAreUniversal2023} are based on the First Fundamental Theorem of $GL(V, n)$, the generalized linear group over a finite-dimensional vector space $V$ of dimension $n$ \cite{kraftCLASSICALINVARIANTTHEORYa}. 

\begin{theorem}{\cite{kraftCLASSICALINVARIANTTHEORYa}} (First Fundamental Theorem for $\mathrm{GL}(V, n)$) \label{thm:fftgl}
The ring of invariants for the action of $\mathrm{GL}(V, n)$ on $V^p \oplus V^{*q}$ is generated by the invariants $(i \mid j)$:
\[
K[V^p \oplus V^{*q}]^{\mathrm{GL}(V, n)} = K[(i \mid j) \mid i = 1, \ldots, p,\, j = 1, \ldots, q].
\]
\end{theorem}

\paragraph{Invariants of vectors and covectors} 
Here, we breifly recall the idea behind the theorem. The theorem is based on the concept of invariants of vectors and covectors. Let $V$ be a finite-dimensional $K$-vector space, for example $K=\mathbb{C}$ and $V = \mathbb{C}^n$. Consider the representation of $\mathrm{GL}(V)$ on the vector space
\[
W := \underbrace{V \oplus \cdots \oplus V}_{p \text{ times}} \oplus \underbrace{V^* \oplus \cdots \oplus V^*}_{q \text{ times}} =: V^p \oplus V^{*q},
\]
consisting of $p$ copies of $V$ and $q$ copies of its dual space $V^*$, given by
\[
g(v_1, \ldots, v_p, \varphi_1, \ldots, \varphi_q) := (gv_1, \ldots, gv_p, g\varphi_1, \ldots, g\varphi_q)
\]
where $g\varphi_i$ is defined by $(g\varphi_i)(v) := \varphi_i(g^{-1}v)$ and $g \in GL(V, n)$. This representation on $V^*$ is the \textit{dual} representation of $\mathrm{GL}(V)$ on $V$, where the elements of $V$ are called \textit{vectors}, while elements of the dual space $V^*$ are called \textit{covectors}. We want to describe the invariants of $V^p \oplus V^{*q}$ under this action. 

For every pair $(i, j)$, $i = 1, \ldots, p$, $j = 1, \ldots, q$, we define the bilinear function $(i \mid j)$ on $V^p \oplus V^{*q}$ by
\[(i \mid j) : (v_1, \ldots, v_p, \varphi_1, \ldots, \varphi_q) \mapsto (v_i \mid \varphi_j) := \varphi_j(v_i).\]
These functions are called \textit{contractions}, and they are invariant to the actions $g \in GL(V, n)$:
\[
(i \mid j)(g(v, \varphi)) = (g\varphi_j)(gv_i) = \varphi_j(g^{-1}gv_i) = (i \mid j)(v, \varphi).
\]
The First Fundamental Theorem (sometimes referred as FFT) states that these functions generate the ring of invariants, i.e. polynomial functions on $V$. 
We first present a result about the universal representation theorem for $GL(n)$-invariant polynomial functions.
\begin{tcolorbox}[title=$GL(n)$ invariance for polynomials]
\begin{theorem}
\label{th:gl_v1}
For a $GL(n)$-invariant polynomial function $f(\bm{x}_1,\dots,\bm{x}_m): X^m \to \mathbb{R}$, with $X \subset \mathbb{R}^n$ a compact space, $f$ can be represented as 
$$
f(\bm{x}_1,\dots,\bm{x}_m) = \sum_{q=1}^{2m^2+1} \psi_q(\sum_{i,j=1}^{m,m} \phi_{qij}( ( \bm{x}_i|\bm{x}_j ) )),
$$
\end{theorem}
\end{tcolorbox}
\begin{proof} 
By \autoref{thm:fftgl}, $f$ is represented by a polynomial function $g$ whose input is $\{(i|j)\}_{i, j}$:
\[f(\bm{x}_1,\dots,\bm{x}_m) = g(\{(i|j)\}_{i, j}).\]
Since $g$ is continuous, we can apply \autoref{th:kan} to represent $g$, which completes the proof.
\end{proof}

In the following, we will focus on the action of some representative subgroups of $GL(n)$ in scientific discovery and prove that the representation theorem shown above could be generalized to continuous (not necessarily polynomial) functions invariant or equivariant to the actions of the subgroups. We note that while we are aware that the result of \cite{ISMAILOV2008113} can further generalize some of the following claims to the case of non-continuous functions, we will leave this generalization for future research, especially for the case of learning non-continuous invariant and equivariant functions. 

\subsection{Invariants for specific metric}
The invariants for specific metrics or symmetric groups are
\begin{itemize}
    \item Euclidean; Poincare
    $$
    (i|j) = \langle \bm{x}_i , \bm{x}_j \rangle = \bm{x}_i^T \bm{x}_j
    $$
    \item Mikowski
    $$
    (i|j) = \langle \bm{x}_i , \bm{x}_j \rangle = \bm{x}_i^T \bm{\Lambda} \bm{x}_j
    $$
    with $\bm{\Lambda}_{ij} = 1/2(1-2 1_{i-1})\delta_{ij}$, i.e. $\bm{\Lambda}_{11} = -1$ and $\bm{\Lambda}_{ii} = 1, i>1$
    \item $GL(V=\mathbb{R}^n)$
    $$
    (i|j) = \langle \bm{x}_i , \bm{x}_j \rangle = \bm{x}_i^T \bm{A} \bm{x}_j
    $$
    
\end{itemize}

\begin{table}[t]
    \begin{mdframed}%
    \centering%
    \begin{tabular}{rl}
    Generalized Linear &
    $\text{GL}(V,n) =\{ M \in \mathbb R^{ n \times n}: M ^\top M = A, \operatorname{det}(A) \neq 0\}, \label{eq.gl}$    
    \\[0.5ex]
    & $M(v_1,\cdots, v_n) = (M\,v_1, \cdots, M\,v_n)$
    \\[0.5ex]    
    Orthogonal &
    $\text{O}(n) =\{ Q \in \mathbb R^{ n \times n}: Q^\top Q = Q\,Q^\top = I_n\}, \label{eq.o}$
    \\[0.5ex]
    & $Q(v_1,\cdots, v_n) = (Q\,v_1, \cdots, Q\,v_n)$
    \\[0.5ex]    
    Rotation &
    $\text{SO}(n) =\{ Q \in \mathbb R^{n\times n}: Q^\top Q = Q\,Q^\top = I_n, \; \label{eq.so} \operatorname{det}(Q)=1\}$
    \\[0.5ex]
    & $Q(v_1,\cdots, v_n) = (Q\,v_1, \cdots, Q\,v_n)$
    \\[0.5ex]    
    Translation &
    $\text{T}(n) =\{ w \in \mathbb R^{n} \}$
    \\[0.5ex]
    & $w(v_1,\cdots, v_n) = (v_1 + w, \cdots, v_n + w)$
    \\[0.5ex] 
    Euclidean &
    $\text{E}(n) = \text{T}(n) \otimes \text{O}(n)$
    \\[0.5ex]
    & $(w,Q)(v_1,\cdots, v_n) = (Q\,v_1 + w, \cdots, Q\,v_n + w)$ 
    \\[0.5ex]    
    Lorentz &
    $\text{O}(1,n-1) =\{ Q \in \mathbb R^{n\times n}: Q^\top\Lambda\, Q =\Lambda, \,\Lambda=\text{diag}([-1,1,\ldots,1]) \}$
    \\[0.5ex]
    & $(w,Q)(v_1,\cdots, v_n) = (Q\,v_1 + w, \cdots, Q\,v_n + w)$ 
    \\[0.5ex]    
    Poincar\'e &
    $\text{IO}(1,d) = \text{T}(n) \otimes \text{O}(1,n-1)$
    \\[0.5ex]
    & $(w,Q)(v_1,\cdots, v_n) = (Q\,v_1 + w, \cdots, Q\,v_n + w)$ 
    \\[0.5ex]    
    Permutation & $\text{S}_n=\{\sigma:[n]\to [n] \text{ bijective function}\}$
    \\[1ex]
    & $\sigma(v_1,\ldots, v_n)=(v_{\sigma(1)},\ldots,v_{\sigma(n)})$
    \\[1ex]    
    \end{tabular}
    \caption{Similar to \cite{villarScalarsAreUniversal2023}, we summarize here the more important symmetries we are considering. Groups ($G$) and the associated actions $g \in G$ of the groups on the elements of the vector space $\bm{v} = (v_1,\dots,v_n) \in V$.}
    \label{tab:group-actions}
    \end{mdframed}
\end{table}

\subsection{ \texorpdfstring{$S_n$}-- Permutation invariance }
\begin{lemma}(Permutation invariance)
\label{lm:perm}
Suppose we have continuous real functions $\psi_{q}, \phi_q: \mathbb{R} \rightarrow \mathbb R$ for $\forall q \in [2m+1]$. Then, the following function is invariant to the action of the permutation group:
\[f(x_1,\dots,x_m) = \sum_{q=1}^{2m+1} \psi_q(\sum_{p=1}^{m} \phi_{q}(x_p)), \ \ x = (x_1, \cdots, x_m) \in \mathbb{R}^m. \]
\end{lemma}

\begin{proof}
Since the decomposition requires the output of the function to not change after a generic permutation $\pi$ of the input, then 
$$
\sum_{q=1}^{2m+1} \psi_{q}(\sum_{p=1}^{m} \phi_{qp}(x_p)) = \sum_{q=1}^{2m+1} \psi_{q}(
\sum_{p=1}^{m} \phi_{qp}(x_{\pi{(p)}})
)
$$
to be true, it is sufficient to drop the dependence of $\phi_{qp}$ on the node index $p$.
\end{proof}

\begin{remark}
We note that while the expression looks similar to KAT, it is not known whether the above expression is universal for arbitrary permutation invariant functions.    
\end{remark}

\subsection{Permutation invariance and its connection to DeepSet}

We present two theorems that connect DeepSet and KAT. \cite{zaheer2017deep} proposes a connection to KAT using high-dimensional functions; this theorem has been extended in \cite{amir2023neural} by considering linear functions.

\begin{theorem}[\textbf{Theorem 7 \cite{zaheer2017deep}}](DeepSet Permutation Invariant representation)
\label{th:deep-set}
\; Let $f: [0, 1]^m \to \mathbb{R}$ be an arbitrary multivariate continuous function iff it has the representation
\begin{equation}
f(x_1, ..., x_m) = \rho\left(\sum_{p=1}^m \phi(x_p)\right)
\end{equation}
with continuous outer and inner functions $\rho:\mathbb{R}^{2m+1}\to\mathbb{R}$ and $\phi:\mathbb{R}\to\mathbb{R}^{2m+1}$. The inner function $\phi$ is independent of the function $f$.
\end{theorem}

\begin{corollary}[\textbf{Corollary 6.1 \cite{amir2023neural}}]
\label{th:neural-injective}
Let $m,d\in \NN$ and set $M=2md+1$. Let $\sigma: \RR \to \RR $ be an analytic non-polynomial function. Let $K\subseteq \RR^d $ be a compact set. Then there exist $\bA\in \RR^{M\times d}, \bb\in \RR^M $ such that for any continuous permutation-invariant $f:K^m \to \RR $, there exists a continuous $F:\RR^M \to \RR$ such that 
\begin{equation}\label{eq:decomposable}
f(\bX)=F\left( \sum_{p=1}^m \sigma(\bA \bx_p+\bb)\right) , \quad \forall \bX=\left(\bx_1,\ldots,\bx_m\right) \in K^m.
\end{equation}
\end{corollary}

We notice that the feature used in \cref{lm:perm} have similarity with the features generated in \cref{th:neural-injective}, therefore we propose the following lemma.

\begin{tcolorbox}[title= Permutation invariance v2]
\begin{lemma} 
\label{lm:prem-v2}
Let $m \in \NN$ and set $M=2m+1$ and let $K\subseteq \RR^d $ be a compact set, then for any continuous permutation-invariant $f:K^m \to \RR $, there exists 
continuous univariate functions $\psi_q, \phi_{qp}, \varphi_{r}$ such that 
\[
f(x_1,\dots,x_m) = F(\sum_{p=1}^{m} \varphi_{1}(x_p), \dots,
\sum_{p=1}^{m} \varphi_{2m+1}(x_p) ), 
\ \ x = (x_1, \cdots, x_m) \in \mathbb{R}^m. 
\]
with 
\[
F(y_1,\dots,y_M) = \sum_{q=1}^{2M+1} \psi_q(\sum_{p=1}^{M} \phi_{qp}(y_p)), \ \ y = (y_1, \cdots, y_M) \in \mathbb{R}^M. 
\]
\end{lemma}
\end{tcolorbox}

\begin{proof}
The proof is based on \cref{th:neural-injective} setting $d=1$ and building the functions
$$
\varphi_{r}(x_p) = \sigma(\bA_r x_p+\bb_r)
$$
where now $\bA, \bb$ are vectors of dimension $M$. 
Then apply KAT \cref{th:kan} to the function $F(y_1,\dots,y_M)$, remembering that the image of a compact set from a continuous and bounded function is compact.
\end{proof}
Compared to \cref{lm:perm}, this version, while stronger, requires $\approx M^3$ applications of univariate functions.

\subsection{ \texorpdfstring{$O(n)$}--invariance} 
\label{annex:OnInvariant}

We here consider the permutation group that acts on the input $(\bm{x}_1, \dots, \bm{x}_m)$ and present the architecture invariant to the action of the orthogonal group.

\begin{tcolorbox}[title=$O(n)$ invariance - v1]
\begin{theorem}
\label{th:on_v1}
For a continuous function invariant to the action of $O(n)$ $f(\bm{x}_1,\dots,\bm{x}_m): X^m \to \mathbb{R}$, with $X \subset \mathbb{R}^n$ a compact space, it can be represented as 
$$
f(\bm{x}_1,\dots,\bm{x}_m) = \sum_{q=1}^{2m^2+1} \psi_q(\sum_{i,j=1}^{m,m} \phi_{qij}(\langle \bm{x}_i,\bm{x}_j \rangle)),
$$
\end{theorem}
\end{tcolorbox}
\begin{proof}
    We first apply Theorem \ref{thm:villar} to get an invariant representation $f = g((\langle x_i,x_j \rangle)_{i,j=1}^m)$ and apply \autoref{th:kan} for $g$ to get a KAT representation, which completes the proof.
\end{proof}

The above invariant representation takes a high computational cost. In the following we give one computationally efficient model. The detail is also described in \autoref{sec:complexity}.
\begin{tcolorbox}[title=$O(n)$ invariance - v2]
\begin{theorem}
\label{th:on_v2}
For a continuous function invariant to the action of $O(n)$ $f(\bm{x}_1,\dots,\bm{x}_m): X^m \to \mathbb{R}$, with $X \subset \mathbb{R}^n$ a compact space, it can be represented as 
$$
f(\bm{x}_1,\dots,\bm{x}_m) = \sum_{q=1}^{2mn+1} \psi_q \left( \sum_{i=1,j=1}^{m,n} \phi_{qij}(\langle \bm{x}_i,\bm{y}_j \rangle)
+  \sum_{i=1,j=1}^{n,n} \phi'_{qij} (\langle \bm{y}_i,\bm{y}_j\rangle)
\right),
$$
where $\bm{y}_j^q = \alpha_j^q (\bm{x}_1,\dots,\bm{x}_m) = \sum_{p=1}^{m} \alpha_p^{j} \bm{x}_p$, with $\bm{y}_j^q$ a linear combination of $\{\bm{x}_p\}$ with scalars $\alpha_{p}$ such that $\text{span}(\{\bm{y}_j^q\}_{j=1}^n)=\mathbb{R}^n$. 
\end{theorem}
\end{tcolorbox}
\begin{proof}
The proof is based on the use of Theorem \ref{thm:villar} , 
\autoref{th:corr}
and \autoref{th:kan}. Since we define $y_j$ as linear combination of $x_p$ then also $\langle x_p,y_j \rangle$ and $\langle y_p,y_j \rangle$ are invariant to rotation, e.g. $\langle R x_p, y'_j \rangle = \langle R x_p,\sum \alpha_i R x_i \rangle = \langle R x_p, R \sum \alpha_i x_i) \rangle = \langle R x_p,R y_j \rangle = \langle x_p,y_j \rangle$. 
\end{proof}

As a corollary, we get the following further compact form when input vectors span $\mathbb{R}^{n}$. The derivation is done in a manner similar to Theorem \ref{th:on_v2} except applying \autoref{th:subset} instead of \autoref{th:corr}:
\begin{tcolorbox}[title=$O(n)$ invariance - v3]
\begin{corollary}(same as \cref{th:on_v3-main})
\label{th:on_v3}
Suppose that  $\operatorname{span}(\{\bm{x}_j\}_{j=1}^n)=\mathbb{R}^n$. Then, a continuous function invariant to the action of $O(n)$ $f(\bm{x}_1,\dots,\bm{x}_m): X^m \to \mathbb{R}$, with $X \subset \mathbb{R}^n$ a compact space, can be represented as 
$$
f(\bm{x}_1,\dots,\bm{x}_m) = \sum_{q=1}^{2mn+1} \psi_q \left( \sum_{i=1,j=1}^{m,n} \phi_{qij}(\langle \bm{x}_i,\bm{x}_j \rangle)
\right).
$$
\end{corollary}
\end{tcolorbox}

\subsection{\texorpdfstring{$O(n)$}\ \ and permutation invariance} \label{annex:on_perm_in}
We further consider the permutation group action to the input $(\bm{x}_1, \dots, \bm{x}_m)$ and present the architecture invariant to the action of the permutation group. 

\begin{corollary}($O(n)$ and $S_n$ permutation invariance - v1)
\label{cor:orth_perm_v1}
The following function is invariant to the action of the permutation group and the orthogonal group $O(n)$: 
\[f(\bm{x}_1,\dots,\bm{x}_m) = \sum_{q=1}^{2mn+1} \psi_q \left( \sum_{i=1,j=1}^{m,m} \phi_{q}(\langle \bm{x}_i,\bm{x}_j \rangle)
\right).\]
\end{corollary}
\begin{proof}
    We based this result on \autoref{th:on_v1} and \Cref{lm:perm}, by removing the dependence on the node index. 
\end{proof}

\begin{remark}
We note that while the expression looks quite similar to KAT in appearance, it is not known whether the above expression is universal for arbitrary $O(n)$ permutation invariant functions.    
\end{remark}
\subsection{\texorpdfstring{$O(n)$}--equivariance} \label{annex:eqivariance} We have the corresponding equivariant version.

\begin{tcolorbox}[title=$O(n)$ equivariance - v1]
\begin{theorem} 
\label{th:on_eq_v1}
Suppose that  $\text{span}(\{\bm{x}_j\}_{j=1}^n)=\mathbb{R}^n$. 
For a continuous function equivariant to the action of $O(n)$ $f(x_1,\dots,x_m): X^m \to X$, with $X \subset \mathbb{R}^n$ compact space, it can be represented as 
$$f(\bm{x}_1,\dots,\bm{x}_m) = \sum_{k=1}^m  \sum_{q=1}^{2mn+1} \psi^k_q \left( \sum_{i=1,j=1}^{m,m} \phi^k_{qij}(\langle \bm{x}_i,\bm{x}_j \rangle)
\right) \bm{x}_k.$$
\end{theorem}
\end{tcolorbox}
\begin{proof}
    We use (Proposition 4) from \cite{villarScalarsAreUniversal2023} to get an equivariant form multiplied with $O(n)$-invariant functions and apply \autoref{th:on_v1} to the respective invariant functions. 
\end{proof}

Similar results can be obtained for the representation from  \autoref{th:on_v2} or \autoref{th:on_v3}. 

It is also possible to show that we can use the gradients of invariant functions to build a generic equivariant function, in particular, if 
$f(\bm{x},\dots,\bm{x}_m)$ is invariant, then 
$
\nabla_{\bm{x}_i} f(\bm{x},\dots,\bm{x}_m)
$
is equivariant, as it is 
$
\sum_{i=1}^m \alpha_i \nabla_{\bm{x}_i} f(\bm{x},\dots,\bm{x}_m).
$
Extending the previous results with these forms is easy when $f$ is decomposed according to \autoref{th:on_v1}, \autoref{th:on_v2} or \autoref{th:on_v3}. 

\subsection{ \texorpdfstring{$O(n)$}-- equivariance and permutation invariance} 
We have the corresponding equivariant and permutation invariant versions.

\begin{corollary} ($O(n)$ equivariance and permutation invariance - v1)
\label{th:on_perm_eq_v1}
The following function is invariant to the action of the permutation group and the orthogonal group $O(n)$: 
$$
f(\bm{x}_1,\dots,\bm{x}_m) = \sum_{i=1}^m  \sum_{q=1}^{2mn+1} \psi_q \left( \sum_{j=1}^{m} \phi_{q}(\langle \bm{x}_i,\bm{x}_j \rangle)
\right) \bm{x}_i.
$$

\end{corollary}

\begin{proof} 
    We based this result on \autoref{th:on_eq_v1} and \Cref{lm:perm}. 
\end{proof}

\subsection{Mapping invariant features}
\begin{lemma} \label{th:A14}
Suppose that we have $\bm{X} \in \R^{m \times n}$ and $\bm{Y} \in \R^{k \times n}$ with $\rho(\bm{Y}) = n, n \le k$ 
then
 $$
 \bm{X}\bm{Y}^T (\bm{Y}\bm{Y}^T)^{\dagger} \bm{Y}\bm{X}^T = \bm{X}\bm{X}^T,  
 $$
where $\rho(X)$ is the matrix rank and $^\dagger$ is the pseudo-inverse.
\end{lemma}
\begin{proof}
The equality follows from these properties:
$$
\bm{Y} = \bm{V} \bm{\Lambda} \bm{U}, ~~~ \bm{V}^T\bm{V} = \bm{I}_k, 
    ~~~ \bm{U}^T\bm{U} = \bm{I}_n = \bm{U}\bm{U}^T,
$$
$$
(\bm{Y}\bm{Y}^T)^{\dagger} = (\bm{V} \bm{\Lambda}  \bm{\Lambda}^T \bm{V}^T)^{\dagger} = \bm{V} ( \bm{\Lambda}  \bm{\Lambda}^T )^{\dagger} \bm{V}^T, \, \, \bm{Y}^T = \bm{U}^T  \bm{\Lambda}^T \bm{V}^T,
$$
$$
\bm{Y}^T (\bm{Y}\bm{Y}^T)^{\dagger} \bm{Y} = \bm{U}^T  \bm{\Lambda}^T \bm{V}^T  \bm{V} ( \bm{\Lambda}  \bm{\Lambda}^T )^{\dagger} \bm{V}^T  \bm{V} \bm{\Lambda} \bm{U} = \bm{U}^T  \bm{\Lambda}^T  ( \bm{\Lambda}  \bm{\Lambda}^T )^{\dagger} \bm{\Lambda} \bm{U} = \bm{I}_n.$$ 
\end{proof}
\vspace{-3mm}
\begin{theorem} (Correlation matrix representation) 
\label{th:corr}
Given $\bm{x}_1,\dots,\bm{x}_m \in \mathbb{R}^n$ and a set of points $\bm{y}_1,\dots,\bm{y}_k \in \R^n$, such that $\rho(\bm{y}_1,\dots,\bm{y}_k) = n$, there is an invertible map between these two sets:
\begin{itemize}
    \item $\{ \langle \bm{x}_i,\bm{x}_j\rangle \}_{i,j=1}^{m,m}$, with a total number of variable equal to $m^2$
    \item $\{ \langle \bm{x}_i,\bm{y}_j\rangle \}_{i,j=1}^{m,k}$, $\{ \langle \bm{y}_i,\bm{y}_j\rangle \}_{i,j=1}^{k,k}$ with a total number of variable equal to $mk+k^2$
\end{itemize}
\end{theorem}
\begin{proof}
    Define $\bm{X} =(\bm{x}_1,\dots,\bm{x}_m)^T \in \R^{m \times n}$ and $\bm{Y} =(\bm{y}_1,\dots,\bm{y}_k)^T \in \R^{k \times n}$
    then 
    $$
    \bm{X}\bm{X}^T = \{ \langle \bm{x}_i,\bm{x}_j\rangle \}_{i,j=1}^{m,m}, \, \, 
    \bm{X}\bm{Y}^T = \{ \langle \bm{x}_i,\bm{y}_j\rangle \}_{i,j=1}^{m,k},\, \text{and}\,
    \bm{Y}\bm{Y}^T = \{ \langle \bm{y}_i,\bm{y}_j\rangle \}_{i,j=1}^{k,k}.
    $$
    We then apply \autoref{th:A14} to yield
     $
     \bm{X}\bm{X}^T  = \bm{X}\bm{Y}^T (\bm{Y}\bm{Y}^T)^{\dagger} \bm{Y}\bm{X}^T.
     $     
    Notice that $\bm{Y}\bm{X}^T = (\bm{X}\bm{Y}^T)^T$, and   
    therefore we have the result.
\end{proof}

We define $\bm{Y}$ as a subset of $\bm{X} \in \R^{m \times n}$ of size $k$, then it is a matrix of dimension $k \times n$, which we ask to have rank $n$. We then can say,
\begin{corollary} (Special case - Subset) \label{th:subset}
If $\bm{Y}=\bm{X}[:n]$, with $n \le k$, $\rho(\bm{Y})=n$, $\bm{X} \in \R^{ m \times n}$, $m \le k \le n$, then there is an invertible map between these two sets:
\begin{itemize}
    \item $\{ \langle \bm{x}_i,\bm{x}_j\rangle \}_{i,j=1}^{m,m}$, with a total number of variable equal to $m^2$
    \item $\{ \langle \bm{x}_i,\bm{x}_j\rangle \}_{i,j=1}^{m,k}$, if $\bm{y}_j = \bm{x}_j$, with a total number of variable equal to $mk$,     
\end{itemize}    
\end{corollary}
\begin{proof}
    We use \autoref{th:corr} and notice that 
    $\bm{Y}\bm{Y}^{\operatorname{T}}$ 
    can be derived from 
    $\bm{Y}\bm{Y}^\mathrm{T} = \bm{X}[:n]\bm{X}[:n]^\mathrm{T}$, 
    which are included in the previous features. 
\end{proof}

\subsection{Computational validation of the main theorem}
There is one step in our theorem that creates concern. This step is as follows: once we change the basis for our data, we build the basis from the data itself. 
We now prove with a simple Python code that this is the case. 

\begin{lstlisting}[language=Python, caption=
% Python based informal proof
Python code to validate the contribution.
]
# some help functions
rot_gen = lambda n: np.linalg.svd(np.random.randn(n,n))[0]
basis = lambda X: X[:n,:]
corr = lambda X: X @ X.T
inv = lambda X,Y: X @ Y.T
rot = lambda X,R: X @ R
#set the seed; it can be removed or changes
np.random.seed(42)
# the problem's dimension, can be changed, but m>=n
m,n = 5,3
# this is my data
X = np.random.randn(m,n)
# the correlation matrix of the data, which is an invariant feature
C1 = corr(X)
# we build a basis that depends on the input
Y = basis(X)
# compujte invariant features
Z1 = inv(X,Y)
# compute the correlation of the new features
D1 = corr(Z1)
# some rotation
R = rot_gen(n)
# apply the rotation to the input
X = rot(X, R)
# rebuild the basis
Y = basis(X)
# compute the invariant features
Z2 = inv(X,Y)
# compute the correlation with the new invariant features
D2 = corr(Z2)
# Question: is the correlation matrix before and after the same (we know is the same):
print(np.linalg.norm(C1 - C2))
# Result: 1.934545700657722e-15 (yes, numerically the same)
# Question: is the correlation matrix with the invariant feature the same before and after (they should)
print(np.linalg.norm(D1 - D2))
# Result: 9.407543438562363e-15 (yes, numerically the same)
# Question: are the invariant feature the same, before and after the rotation (they better be)?
print(np.linalg.norm(Z1 - Z2))
# Result: 1.4220500840710913e-15 (yes, numerically the same)
\end{lstlisting}

\subsection{Computational validation of \autoref{th:A14}}
\begin{lstlisting}[language=Python, caption=Python code to validate the contribution.]
import numpy as np
from numpy.linalg import norm
np.random.seed(42)
# the problem's dimension, can be changed, but m>=n
m,n = 15,3
k = n+2
# create the two matricies
X = np.random.randn(m,n)
Y = np.random.randn(k,n)
# Verify Theorem A.14
print(norm(X @ Y.T @ np.linalg.pinv(Y @ Y.T) @ Y @ X.T - X@X.T))
# Result: 1.2816111681783468e-14
\end{lstlisting}

\section{Complexity} \label{sec:complexity}
The representation complexity of \autoref{eq:inv_v1} is $O(m^4)$, which is quite larger than the complexity we have if we apply KAT directly to the coordinates of the nodes, i.e. $O(m^2n^2)$, which ignores the symmetries of the problem. However, in \autoref{eq:inv_v3}, we show that we can represent the invariant function $f$ with complexity $O(m^2 n^2)$, thus similar to the non-invariant KAT.

\section{Additional Experiments}

\begin{figure}[ht]
    \centering
    \includegraphics[width=0.55\linewidth]{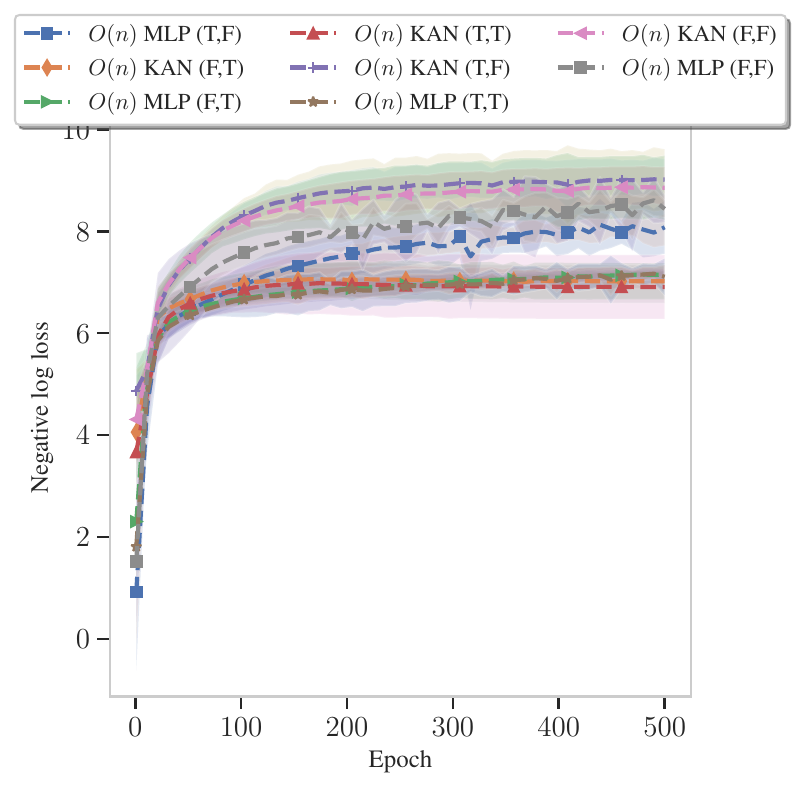}
    \caption{Test performance (Negative log Huber Loss) of various models for the linear polymers. $O(n)$ is the model that is invariant to rotation and reflection on $\mathbb{R}^n$, while $\pi$ is the permutation invariant model. In parenthesis, the two flags indicate if the model includes the node index and the second if the features are linear or quadratic in the number of nodes. 
    }
    
    \label{fig:lp_te_false}
\end{figure}

\begin{figure}[ht]
    \centering
    \includegraphics[width=0.55\linewidth]{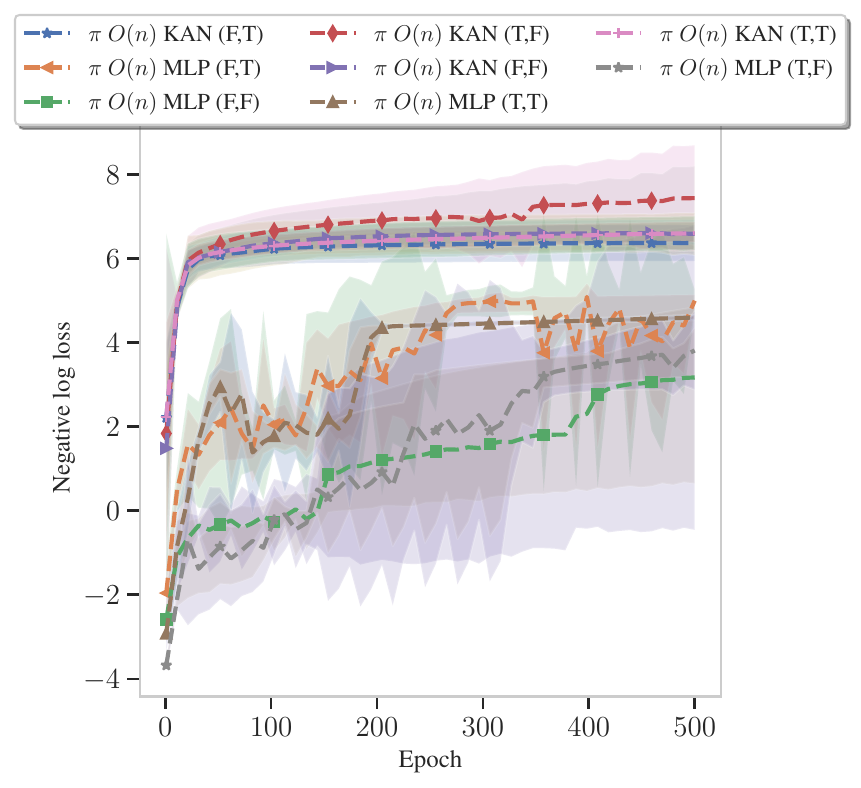}     \caption{Test performance (Negative log Huber Loss) of various models for the linear polymers. $O(n)$ is the model that is invariant to rotation and reflection on $\mathbb{R}^n$, while $\pi$ is the permutation invariant model. In parentheses, the two flags indicate if the model includes the node index and the second if the features are linear or quadratic in the number of nodes. 
    }
    \label{fig:lp_te_true}
\end{figure}

\begin{table}
\caption{Huber NLL for the Linear Polymer dataset, with $a_i=0$ on different dimensions ($3,5$) and different number of nodes $4,10,15$. }
\label{tab:lp1}
\begin{tabularx}{\columnwidth}{
@{} Y{1.5} Y{0.9} Y{0.9} Y{0.9} Y{0.9} @{}
}
\toprule
LinPoly-1 & $O(n)$ KAN & $O(n)$ MLP & $\pi~O(n)$ KAN & $\pi~O(n)$ MLP \\
\midrule
m4/n3 & 10.85 & 11.74 & 9.07 & 9.29 \\
 & 0.52 & 0.18 & 0.40 & 0.33 \\
m10/n3 & 8.93 & 9.08 & 7.36 & 6.40 \\
 & 0.33 & 0.36 & 0.32 & 0.25 \\
m10/n5 & 9.22 & 9.06 & 7.41 & 5.97 \\
 & 0.10 & 0.16 & 0.14 & 0.12 \\
m15/n3 & 7.99 & 7.98 & 6.91 & 4.82 \\
 & 0.34 & 0.24 & 0.38 & 0.51 \\
m15/n5 & 7.99 & 7.81 & 6.76 & 4.18 \\
 & 0.33 & 0.16 & 0.39 & 0.78 \\
\bottomrule
\end{tabularx}
\end{table}

\subsection{Linear Polymer experiments}
\label{sec:linear-polymer}
Linear polymers are chain molecules composed of repeating structural units (monomers) linked together sequentially. Linear polymers exhibit flexibility and thermoplastic behavior. Examples include polyethylene (PE), polyvinyl chloride (PVC), and polystyrene (PS), and find applications in packaging, textiles, and plastic films due to their ease of processing, recyclability, and ability to be melted and reshaped.
\autoref{fig:lp_te_false} and \autoref{fig:lp_te_true} show the performance with $O(n)$ symmetry and with additionally permutation symmetry. 
Additional details in \autoref{annex:polymers}.

\begin{figure*}
    \centering
    \subfigure[]{
        \includegraphics[width=.2\textwidth]{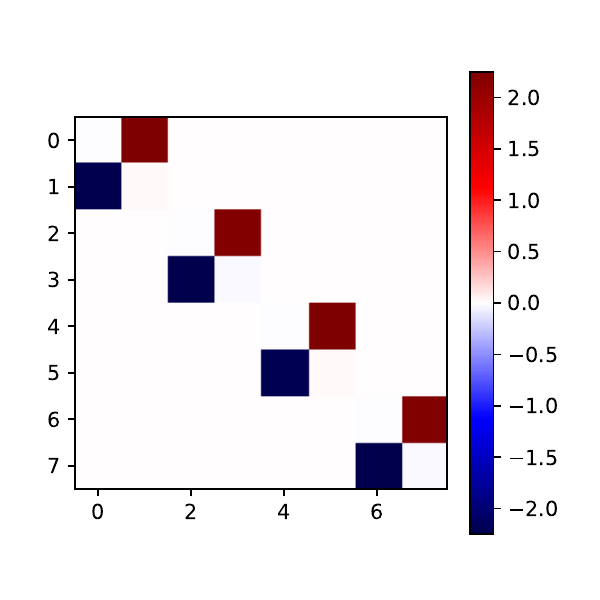} 
    }%
    \hspace{1cm}
    \subfigure[]{
        \includegraphics[width=0.2\textwidth]{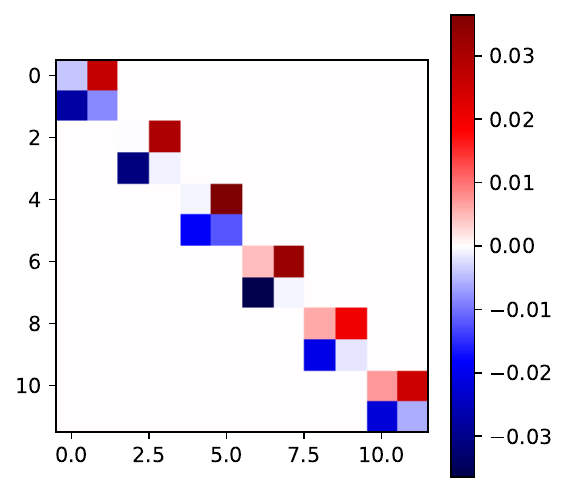} 
    }%
    \hspace{1cm}
    \subfigure[]{
        \includegraphics[width=0.2\textwidth]{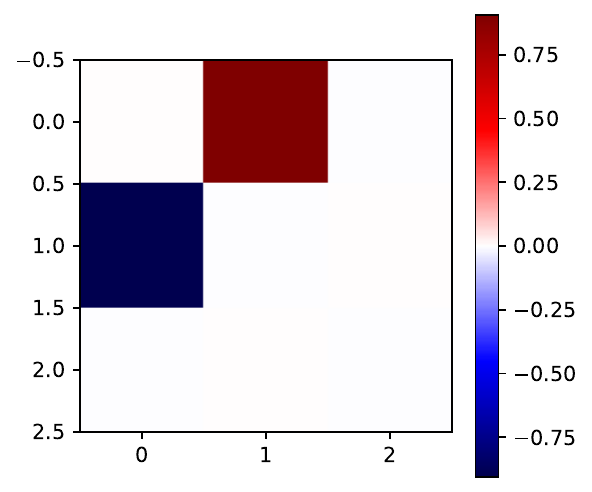} 
    }%
    
    \vspace{\baselineskip} 
    
    \subfigure[]{
        \includegraphics[width=0.2\textwidth]{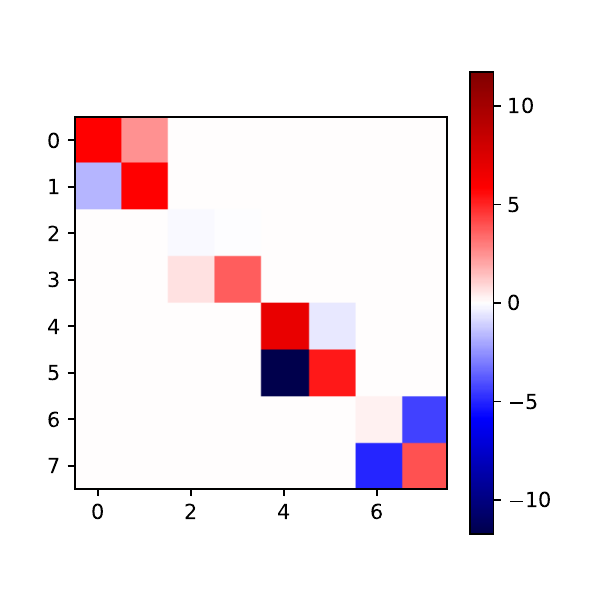} 
    }%
    \hspace{1cm}
    \subfigure[]{
        \includegraphics[width=0.2\textwidth]{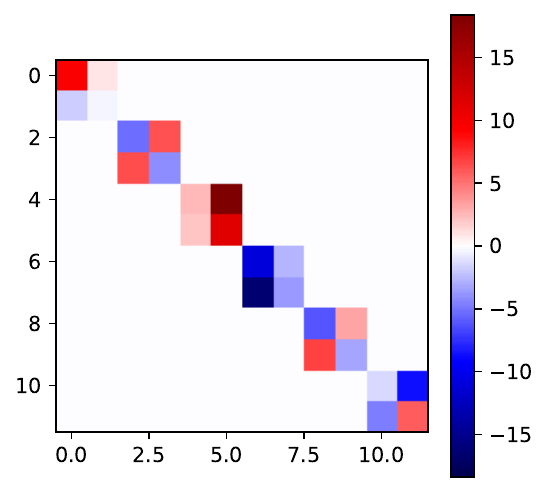} 
    }%
    \hspace{1cm}
    \subfigure[]{
        \includegraphics[width=0.2\textwidth]{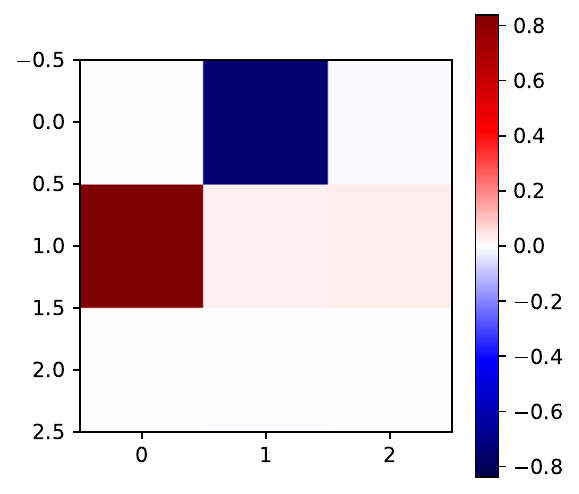} 
    }%
    \caption{Visualization of discovered symmetries. From left to right: $2$-body, $3$ -body and discrete permutation tasks; the top plots are for the LiGAN method and the bottom plots for LieGAN-\method{}. }
    \label{fig:symmetry-discovery}
\end{figure*}

\subsection{Symmetry discovery}
\label{sec:symmetry-discovery}
To test symmetry discovery, we extend the LiGAN \cite{yang2023generative} from \cite{yang2023generative} with \method{} and test on:
1) N-body trajectories; 2) Discrete Rotations.
The learned metric matrices are shown in \cref{fig:symmetry-discovery}, where HNN from 
\cite{greydanus2019hamiltonian} was used by \cite{yang2023generative}. 

\paragraph{N-Body Trajectory}
The task is to discover symmetry from 2-body and  3-body trajectory prediction

\paragraph{Discrete Rotation Invariant Regression}
The task is to discover symmetry for a discrete rotation.

\section{Experiments}\label{annex:experiments}

The LJ task aims to train the energy from atomistic configurations ($3$d conformations), where the configurations are generated using random positions, followed by energy minimization.

The MD task aims at training the atomistic energy from DFT energy computations. The MD datasets are generated using molecular dynamics, therefore, they represent more realistic configurations of the atomic systems.

\begin{table}
\caption{Huber NLL for the LJ-2 dataset}
\centering
\begin{tabular}{lllll}
\toprule
LJ (2) & $O(n)$ KAN & $O(n)$ MLP & $\pi~O(n)$ KAN & $\pi~O(n)$ MLP \\
\midrule
m4/n3 & 9.54 & 8.52 & 9.35 & 8.56 \\
& 0.82 & 0.43 & 0.62 & 0.39 \\
m10/n3 & 8.66 & 8.22 & 8.49 & 5.73 \\
 & 0.66 & 0.61 & 0.74 & 0.25 \\
m10/n5 & 7.52 & 7.02 & 7.19 & 4.84 \\
 & 0.27 & 0.10 & 0.32 & 0.19 \\
m15/n3 & 9.45 & 9.35 & 9.89 & 3.91 \\
 & 1.32 & 1.43 & 2.11 & 0.79 \\
m15/n5 & 6.66 & 6.47 & 6.74 & 2.36 \\
 & 0.23 & 0.27 & 0.25 & 1.33 \\
\bottomrule
\end{tabular}
\end{table}

\subsection{Lennard-Jonnes}\label{annex:LJ}
For the Lennard-Jonnes (LJ) experiments, we generate $m$ particles in $n$ dimensional space. The interaction between particles is described by the LJ potential, 
$$
U_\text{LJ}(r) =  f((a/r)^{12} - (a/r)^6)
$$
where $r$ is the distance between two particles and $a$ is a parameter that defines the minimum energy of the interaction, while $f(x) = x+ \sum_{l=1}^{3} a_l \sin(w_l x)$, with $a_1=1,a_2=.3,a_3=.1, w_1=11,w_2=30,w_3=50$ (or $a_1=a_2=a_3=0$), is an oscillatory term.
After generating the particles, we perform an energy minimization step to relax the system towards a lower energy state, avoiding large energy contributions caused by the random initialization of the particle positions.

\subsection{Linear polymers}\label{annex:polymers}
As an additional experiment, we consider linear polymers of size $m$. The particles are connected to the previous and the following particle by a bond. The interaction between the bond depends quadratically on the difference between the current distance and the desired distance,
$$
U_\text{bond}(r) = f(\| d - \hat{d} \|^2) + U_\text{LJ}(r)
$$
and $f(x) = x+ \sum_{l=1}^{3} a_l \sin(w_l x)$ is an oscillatory term.
For the unbonded particle, the LJ potential is used, as before. 

\begin{table}
\caption{Huber NLL for the LinPoly-2 dataset}
\centering
\begin{tabular}{lllll}
\toprule
LinPoly-2 & $O(n)$ KAN & $O(n)$ MLP & $\pi~O(n)$ KAN & $\pi~O(n)$ MLP \\
\midrule
m4/n3 & 10.51 & 8.78 & 8.41 & 7.13 \\
 & 0.17 & 0.13 & 0.27 & 0.08 \\
m10/n3 & 8.30 & 7.50 & 7.26 & 4.73 \\
 & 0.40 & 0.18 & 0.39 & 0.78 \\
m10/n5 & 8.36 & 7.77 & 7.18 & 4.11 \\
 & 0.45 & 0.16 & 0.48 & 0.64 \\
m15/n3 & 7.40 & 7.47 & 6.95 & 2.98 \\
 & 0.42 & 0.32 & 0.48 & 0.99 \\
m15/n5 & 7.45 & 7.54 & 6.94 & 2.54 \\
 & 0.45 & 0.17 & 0.56 & 1.00 \\
\bottomrule
\end{tabular}
\end{table}

\subsection{MD17}

\autoref{tab:md17_detailed} shows in detail the performance of the different models on the MD17 dataset. 

\begin{table*}
\caption{Huber NLL for the MD17 dataset}
\label{tab:md17_detailed}
\scalebox{0.875}{
\begin{minipage}{\linewidth}

\begin{tabularx}{1.1\textwidth}{
@{} Y{2.0} Y{0.9} Y{0.9} Y{0.9} Y{0.9} 
Y{0.9} Y{0.9} Y{0.9} Y{0.9} Y{0.9} Y{0.9} 
 Y{0.9}  Y{0.9}   Y{0.9}   Y{0.9}   Y{0.9}   Y{0.9} 
@{}
}
\toprule
Dataset (MD17) & aspirin &  & benzene2017 &  & ethanol &  & malonaldehyde &  & naphthalene &  & salicylic &  & toluene &  & uracil &  \\
\midrule
$O(n)$ KAN (F,F) & 6.77 & 0.16 & 8.02 & 0.09 & 7.94 & 0.04 & 7.84 & 0.03 & 7.42 & 0.04 & 7.54 & 0.14 & 7.60 & 0.21 & 8.03 & 0.13 \\
$O(n)$ KAN (F,T) & 6.08 & 0.01 & 7.29 & 0.03 & 7.14 & 0.03 & 7.12 & 0.04 & 6.29 & 0.08 & 6.41 & 0.03 & 6.50 & 0.06 & 7.08 & 0.00 \\
$O(n)$ KAN (T,F) & 6.83 & 0.20 & 8.06 & 0.15 & 8.06 & 0.04 & 7.90 & 0.07 & 7.39 & 0.14 & 7.53 & 0.13 & 7.54 & 0.18 & 8.04 & 0.11 \\
$O(n)$ KAN (T,T) & 6.09 & 0.04 & 7.27 & 0.06 & 7.13 & 0.03 & 7.16 & 0.07 & 6.30 & 0.03 & 6.36 & 0.05 & 6.54 & 0.06 & 7.01 & 0.09 \\
$O(n)$ MLP (F,F) & 5.63 & 0.00 & 5.98 & 0.02 & 5.47 & 0.01 & 5.40 & 0.01 & 5.37 & 0.00 & 5.65 & 0.00 & 5.69 & 0.02 & 5.70 & 0.01 \\
$O(n)$ MLP (F,T) & 5.63 & 0.01 & 5.91 & 0.00 & 5.46 & 0.03 & 5.42 & 0.02 & 5.34 & 0.00 & 5.62 & 0.01 & 5.71 & 0.00 & 5.61 & 0.00 \\
$O(n)$ MLP (T,F) & 5.61 & 0.01 & 5.93 & 0.02 & 5.41 & 0.01 & 5.37 & 0.01 & 5.36 & 0.01 & 5.61 & 0.00 & 5.64 & 0.04 & 5.69 & 0.01 \\
$O(n)$ MLP (T,T) & 5.61 & 0.00 & 5.90 & 0.00 & 5.43 & 0.01 & 5.38 & 0.01 & 5.33 & 0.00 & 5.61 & 0.01 & 5.68 & 0.01 & 5.61 & 0.00 \\
$\pi~O(n)$ KAN (F,F) & 5.68 & 0.02 & 6.73 & 0.18 & 5.95 & 0.18 & 5.83 & 0.04 & 5.82 & 0.10 & 5.81 & 0.14 & 6.11 & 0.11 & 6.21 & 0.11 \\
$\pi~O(n)$ KAN (F,T) & 5.69 & 0.02 & 6.27 & 0.10 & 6.24 & 0.13 & 5.91 & 0.01 & 5.65 & 0.06 & 5.82 & 0.00 & 5.96 & 0.12 & 5.94 & 0.14 \\
$\pi~O(n)$ KAN (T,F) & 5.69 & 0.02 & 6.69 & 0.19 & 6.01 & 0.07 & 5.80 & 0.03 & 5.84 & 0.10 & 5.90 & 0.16 & 6.06 & 0.11 & 6.32 & 0.05 \\
$\pi~O(n)$ KAN (T,T) & 5.69 & 0.02 & 6.34 & 0.20 & 6.15 & 0.16 & 5.87 & 0.07 & 5.59 & 0.11 & 5.80 & 0.10 & 5.98 & 0.05 & 5.93 & 0.16 \\
$\pi~O(n)$ MLP (F,F) & 4.28 & 0.39 & 5.73 & 0.10 & 5.55 & 0.08 & 5.41 & 0.05 & 5.07 & 0.16 & 5.27 & 0.03 & 5.41 & 0.06 & 5.58 & 0.07 \\
$\pi~O(n)$ MLP (F,T) & 5.45 & 0.05 & 5.77 & 0.05 & 5.49 & 0.01 & 5.40 & 0.03 & 5.29 & 0.02 & 5.53 & 0.06 & 5.64 & 0.04 & 5.58 & 0.02 \\
$\pi~O(n)$ MLP (T,F) & 3.84 & 0.59 & 5.47 & 0.26 & 5.44 & 0.01 & 5.34 & 0.04 & 3.08 & 2.83 & 4.41 & 0.82 & 5.07 & 0.27 & 5.47 & 0.03 \\
$\pi~O(n)$ MLP (T,T) & 5.34 & 0.06 & 5.58 & 0.11 & 5.46 & 0.01 & 5.37 & 0.03 & 5.17 & 0.03 & 5.48 & 0.07 & 5.49 & 0.05 & 5.45 & 0.08 \\
\bottomrule
\end{tabularx}
\end{minipage}}
\end{table*}

\subsection{MD22}
\autoref{tab:md22_detailed} shows in detail the performance of the different models on the MD22 dataset. 

\begin{table*}
\caption{Huber NLL for the MD22 dataset}
\label{tab:md22_detailed}
\begin{tabularx}{\textwidth}{
@{} Y{1.5} Y{0.9} Y{0.9} Y{0.9} Y{0.9} 
Y{0.9} Y{0.9} Y{0.9} Y{0.9} Y{0.9} Y{0.9} 
 Y{0.9}  Y{0.9} 
@{}
}
\toprule
Dataset (MD22) & AT-AT-CG-CG &  & AT-AT & & Ac-Ala3-NHMe &  & DHA & & buckyball-catcher &  & stachyose &  \\
\midrule
$O(n)$ KAN (F,F) & NaN & NaN & 7.42 & 0.30 & 5.95 & 0.07 & 5.71 & 0.10 & NaN & NaN & NaN & NaN \\
$O(n)$ KAN (F,T) & 7.94 & 0.19 & 7.27 & 0.18 & 5.65 & 0.03 & 5.59 & 0.01 & 8.92 & 0.20 & 6.24 & 0.11 \\
$O(n)$ KAN (T,F) & NaN & NaN & 7.20 & 0.33 & 5.85 & 0.13 & 5.67 & 0.14 & NaN & NaN & NaN & NaN \\
$O(n)$ KAN (T,T) & 8.10 & 0.09 & 7.38 & 0.05 & 5.64 & 0.04 & 5.56 & 0.01 & 8.77 & 0.28 & 6.36 & 0.13 \\
$O(n)$ MLP (F,F) & 7.66 & 0.07 & 6.57 & 0.01 & 5.58 & 0.00 & 5.53 & 0.00 & 7.30 & 0.00 & 5.74 & 0.00 \\
$O(n)$ MLP (F,T) & 7.64 & 0.03 & 6.57 & 0.01 & 5.58 & 0.00 & 5.51 & 0.00 & 7.27 & 0.00 & 5.70 & 0.01 \\
$O(n)$ MLP (T,F) & 7.57 & 0.04 & 6.56 & 0.01 & 5.57 & 0.01 & 5.53 & 0.00 & 7.25 & 0.02 & 5.70 & 0.02 \\
$O(n)$ MLP (T,T) & 7.59 & 0.05 & 6.55 & 0.03 & 5.56 & 0.00 & 5.51 & 0.00 & 7.27 & 0.01 & 5.67 & 0.00 \\
$\pi~O(n)$ KAN (F,F) & NaN & NaN & 6.60 & 0.02 & 5.58 & 0.01 & 5.49 & 0.00 & NaN & NaN & NaN & NaN \\
$\pi~O(n)$ KAN (F,T) & 7.71 & 0.04 & 6.61 & 0.07 & 5.56 & 0.00 & 5.51 & 0.01 & 7.43 & 0.08 & 5.70 & 0.02 \\
$\pi~O(n)$ KAN (T,F) & 7.75 & NaN & 6.61 & 0.02 & 5.57 & 0.00 & 5.50 & 0.01 & NaN & NaN & 5.77 & NaN \\
$\pi~O(n)$ KAN (T,T) & 7.74 & 0.07 & 6.64 & 0.03 & 5.56 & 0.01 & 5.52 & 0.01 & 7.39 & 0.06 & 5.73 & 0.03 \\
$\pi~O(n)$ MLP (F,F) & NaN & NaN & NaN & NaN & -0.04 & 1.53 & -0.59 & 0.11 & NaN & NaN & NaN & NaN \\
$\pi~O(n)$ MLP (F,T) & 1.25 & 0.36 & 2.10 & 0.36 & 3.76 & 0.20 & 2.09 & 0.35 & 0.61 & 1.21 & 1.27 & 1.82 \\
$\pi~O(n)$ MLP (T,F) & NaN & NaN & -1.29 & 0.22 & 0.23 & 0.26 & -2.68 & 1.40 & NaN & NaN & NaN & NaN \\
$\pi~O(n)$ MLP (T,T) & 0.39 & 0.28 & 1.64 & 0.62 & 1.99 & 2.35 & 1.35 & 1.43 & -0.18 & 0.20 & 1.46 & 1.01 \\
\bottomrule
\end{tabularx}
\end{table*}

\section{Model parameters}

\begin{table}[]
    \centering
    \begin{tabular}{c|c| p{5cm}}
        \toprule
        \textbf{Parameter} & \textbf{Value} & \textbf{Comment} \\
        \midrule
        Number of epochs & 500  & We use 500 for the MD17 and MD22, while 1000 for the LJ experiments\\
        batch size & 4092 & \\
        loss & Huber & We selected Huber, compared to MSE, since it enables better training \\
        em lr & 0.01,  & learning rate for energy minimization for LJ experiments \\
        em niters  & 500 & number of steps for energy minimization for LJ experiments \\
        learning rate & 0.001 & we experimented with multiple rate and fix this for all experiments \\
        num samples  & 10000 & We fix the number of samples, if the dataset contains more data, we first permute the data (same for all experiments) and select the first 10000 samples. \\
        trsamples & 8000 & we split 80/20 training and testing \\
        optimizer & AdamW & \\
        weight decay & $1e-9$ & Weight decay is used to stabilize the training\\
        scheduler & ReduceLROnPlateau & The scheduler helps with different system requirement\\
        KAN layers & [input dim, 16, 16, 1] & the architecture size has been selected in the hyper-parameter search\\
        KAN orders &  [8,8,8] & This is the number of basis per function\\
        KAN Basis & ReLU& While KAN networks use Spline as basis, we experimented with ReLU, GeLU, Sigmoid, and Chebichev Polynomial, ReLU provided the most reliable solution\\
        MLP layers & [input dim, 128, 128, 1]  & the architecture size has been selected in the hyper-parameter search\\
         \bottomrule
    \end{tabular}
    \caption{Hyper-parameters used during training}
    \label{tab:hyper-params}
\end{table}

\subsection{Hyper-parameters and Hyper-parameter search}
\autoref{tab:hyper-params} show the hyper-parameters used during training for the MLP and KAN-based architectures. We implemented a separate hyper-parameter search on both MLP and KAN architecture based on the synthetic dataset, we tested the different sizes of architecture: small (128/16), medium (256/32), and large (512/64); and selected the small for both systems. 

While KAN networks use Spline as the basis, we experimented with ReLU, GeLU, Sigmoid, and Chebichev Polynomial, ReLU provided the most reliable solution across test cases. 

\subsection{LJ}
\autoref{tab:num_param_lj43} shows the number of parameters per model for the LJ experiments with $m=4$ and $n=3$. The impact of the presentation is already visible. KAN is always smaller. 
\autoref{tab:num_param_lj153} and \autoref{tab:num_param_lj155} show the network size for $m=15$ and $n=3,5$. As the input increases the KAN has more parameters than the equivalent MLP.

\begin{table}
\caption{Network sizes during the $4/3$ experiments}
\centering
\label{tab:num_param_lj43}
\begin{tabular}{lllr}
\toprule
system & model & options & size \\
\midrule
m4/n3 & $O(n)$ KAN & FF & 9911 \\
m4/n3 & $O(n)$ KAN & FT & 9911 \\
m4/n3 & $O(n)$ KAN & TF & 12044 \\
m4/n3 & $O(n)$ KAN & TT & 12044 \\
m4/n3 & $O(n)$ MLP & FF & 22145 \\
m4/n3 & $O(n)$ MLP & FT & 22145 \\
m4/n3 & $O(n)$ MLP & TF & 23681 \\
m4/n3 & $O(n)$ MLP & TT & 23681 \\
m4/n3 & $\pi O(n)$ KAN & FF & 4167 \\
m4/n3 & $\pi O(n)$ KAN & FT & 4167 \\
m4/n3 & $\pi O(n)$ KAN & TF & 4475 \\
m4/n3 & $\pi O(n)$ KAN & TT & 4475 \\
m4/n3 & $\pi O(n)$ MLP & FF & 17665 \\
m4/n3 & $\pi O(n)$ MLP & FT & 17665 \\
m4/n3 & $\pi O(n)$ MLP & TF & 17921 \\
m4/n3 & $\pi O(n)$ MLP & TT & 17921 \\
\bottomrule
\end{tabular}
\end{table}

\begin{table}
\caption{Network sizes during the $15/3$ experiments}
\label{tab:num_param_lj153}
\centering
\begin{tabular}{lllr}
\toprule
system & model & options & size \\
\midrule
m15/n3 & $O(n)$ KAN & FF & 250887 \\
m15/n3 & $O(n)$ KAN & FT & 63691 \\
m15/n3 & $O(n)$ KAN & TF & 371803 \\
m15/n3 & $O(n)$ KAN & TT & 87687 \\
m15/n3 & $O(n)$ MLP & FF & 110849 \\
m15/n3 & $O(n)$ MLP & FT & 51713 \\
m15/n3 & $O(n)$ MLP & TF & 137729 \\
m15/n3 & $O(n)$ MLP & TT & 61697 \\
m15/n3 & $\pi O(n)$ KAN & FF & 4167 \\
m15/n3 & $\pi O(n)$ KAN & FT & 4167 \\
m15/n3 & $\pi O(n)$ KAN & TF & 4475 \\
m15/n3 & $\pi O(n)$ KAN & TT & 4475 \\
m15/n3 & $\pi O(n)$ MLP & FF & 17665 \\
m15/n3 & $\pi O(n)$ MLP & FT & 17665 \\
m15/n3 & $\pi O(n)$ MLP & TF & 17921 \\
m15/n3 & $\pi O(n)$ MLP & TT & 17921 \\
\bottomrule
\end{tabular}
\end{table}

\begin{table}
\caption{Network sizes during the $15/5$ experiments}
\label{tab:num_param_lj155}
\centering
\begin{tabular}{lllr}
\toprule
system & model & options & size \\
\midrule
m15/n5 & $O(n)$ KAN & FF & 250887 \\
m15/n5 & $O(n)$ KAN & FT & 111906 \\
m15/n5 & $O(n)$ KAN & TF & 371803 \\
m15/n5 & $O(n)$ KAN & TT & 159216 \\
m15/n5 & $O(n)$ MLP & FF & 110849 \\
m15/n5 & $O(n)$ MLP & FT & 70529 \\
m15/n5 & $O(n)$ MLP & TF & 137729 \\
m15/n5 & $O(n)$ MLP & TT & 85889 \\
m15/n5 & $\pi O(n)$ KAN & FF & 4167 \\
m15/n5 & $\pi O(n)$ KAN & FT & 4167 \\
m15/n5 & $\pi O(n)$ KAN & TF & 4475 \\
m15/n5 & $\pi O(n)$ KAN & TT & 4475 \\
m15/n5 & $\pi O(n)$ MLP & FF & 17665 \\
m15/n5 & $\pi O(n)$ MLP & FT & 17665 \\
m15/n5 & $\pi O(n)$ MLP & TF & 17921 \\
m15/n5 & $\pi O(n)$ MLP & TT & 17921 \\
\bottomrule
\end{tabular}
\end{table}

\subsection{MD17}

\autoref{tab:num_param_md17} shows the number of parameters for the models used in the experiments. The permutation invariant version reduces the need for parameters considerably.

\begin{table}
\caption{Network sizes during the aspirin experiments}
\label{tab:num_param_md17}
\centering
\begin{tabular}{lllr}
\toprule
dataset & model & options & size \\
\midrule
aspirin & $O(n)$ KAN & FF & 1186625 \\
aspirin & $O(n)$ KAN & FT & 147811 \\
aspirin & $O(n)$ KAN & TF & 1692200 \\
aspirin & $O(n)$ KAN & TT & 197535 \\
aspirin & $O(n)$ MLP & FF & 258689 \\
aspirin & $O(n)$ MLP & FT & 82433 \\
aspirin & $O(n)$ MLP & TF & 312449 \\
aspirin & $O(n)$ MLP & TT & 97025 \\
aspirin & $\pi O(n)$ KAN & FF & 4475 \\
aspirin & $\pi O(n)$ KAN & FT & 4475 \\
aspirin & $\pi O(n)$ KAN & TF & 4783 \\
aspirin & $\pi O(n)$ KAN & TT & 4783 \\
aspirin & $\pi O(n)$ MLP & FF & 17921 \\
aspirin & $\pi O(n)$ MLP & FT & 17921 \\
aspirin & $\pi O(n)$ MLP & TF & 18177 \\
aspirin & $\pi O(n)$ MLP & TT & 18177 \\
\bottomrule
\end{tabular}
\end{table}

\subsection{MD22}

As for the MD17 dataset, also for MD22, \autoref{tab:num_param_md22} shows the number of parameters for the models used in the experiments. The permutation invariant version reduces the need for parameters considerably.
\begin{table}
\caption{Network sizes during the AT-AT-CG-CG experiments}
\label{tab:num_param_md22}
\centering
\begin{tabular}{lllr}
\toprule
dataset & model & options & size \\
\midrule
AT-AT-CG-CG & $O(n)$ KAN & FF & 974480535 \\
AT-AT-CG-CG & $O(n)$ KAN & FT & 2938488 \\
AT-AT-CG-CG & $O(n)$ KAN & TF & 1453151821 \\
AT-AT-CG-CG & $O(n)$ KAN & TT & 4256886 \\
AT-AT-CG-CG & $O(n)$ MLP & FF & 7969025 \\
AT-AT-CG-CG & $O(n)$ MLP & FT & 417665 \\
AT-AT-CG-CG & $O(n)$ MLP & TF & 9736193 \\
AT-AT-CG-CG & $O(n)$ MLP & TT & 506753 \\
AT-AT-CG-CG & $\pi O(n)$ KAN & FT & 4475 \\
AT-AT-CG-CG & $\pi O(n)$ KAN & TF & 4783 \\
AT-AT-CG-CG & $\pi O(n)$ KAN & TT & 4783 \\
AT-AT-CG-CG & $\pi O(n)$ MLP & FF & 17921 \\
AT-AT-CG-CG & $\pi O(n)$ MLP & FT & 17921 \\
AT-AT-CG-CG & $\pi O(n)$ MLP & TF & 18177 \\
AT-AT-CG-CG & $\pi O(n)$ MLP & TT & 18177 \\
\bottomrule
\end{tabular}
\end{table}

\section{Ablation study in invariants}
\label{sec:ablation}

\begin{table}
\caption{Hubert NLL or the Buckyball-catcher system of the MD22 dataset, with the linear version of the representation and with the node id for the $\pi O(n)$ KAN model.}
\label{tab:ablation}
\centering
\begin{tabular}{lrrrr}
\toprule
Features & Train NLL &  & Test NLL &  \\
\midrule
cos & 6.80 & $\pm$0.12 & 6.49 & $\pm$0.03 \\
sin-cos & 6.67 & $\pm$0.08 & 5.97 & $\pm$0.62 \\
n1 & 6.77 & $\pm$0.07 & 6.63 & $\pm$0.05 \\
n1-n12 & 6.70 &$\pm$0.16 & 6.48 & $\pm$0.20 \\
n12 & 6.65 & $\pm$0.02 & 6.47 & $\pm$0.11 \\
inner & 6.69 & $\pm$0.00 & 4.69 & $\pm$2.50 \\
inner-n1 & 6.79 & $\pm$0.01 & 6.64 & $\pm$0.09 \\
inner-n1-n12 & 6.65 & $\pm$0.25 & 6.45 & $\pm$0.20 \\
inner-outer & 6.82 & $\pm$0.03 & {\bf 6.67} & $\pm$0.00 \\
inner-outer-n1 & 6.58 & $\pm$0.21 & 6.48 & $\pm$0.27 \\
inner-outer-n1-n12 & 6.82 & $\pm$0.01 & {\bf 6.67} & $\pm$0.02 \\
\bottomrule
\end{tabular}
\end{table}

\cref{tab:ablation} shows the effect of using different invariant features on the performance in terms of NLL for the Buckyball-catcher system of the MD22 dataset. 

We first define some quantities:
\begin{align*}
\|\bm{x}_i \otimes \bm{y}_j \| &= \sqrt{\|\bm{x}_i \|^2 \|\bm{y}_j\|^2-\langle \bm{x}_i,\bm{y}_j \rangle^2} \\
\overline{ \|\bm{x}_i \otimes \bm{y}_j \|}  &= \|\bm{x}_i \otimes \bm{y}_j \| /(\|\bm{x}_i\|\|\bm{y}_j\|), \\
\overline{ \langle \bm{x}_i,\bm{y}_j \rangle}  &= \langle \bm{x}_i,\bm{y}_j \rangle /(\|\bm{x}_i\|\|\bm{y}_j\|), \\
\|\bm{x}_i \otimes \bm{y}_j \| &= \sqrt{\|\bm{x}_i \|^2 \|\bm{y}_j\|^2-\langle \bm{x}_i,\bm{y}_j ,\rangle^2} 
\end{align*}

We can now define the features used as input to the representation, which are:
\begin{align*}
\text{n1: } &  \|\bm{x}_i \|, \|\bm{y}_j\|, \\
\text{n12: } & \|\bm{x}_i-\bm{y}_j\|, \\
\text{inner: } & \langle \bm{x}_i,\bm{y}_j \rangle, \\
\text{outer: } & \|\bm{x}_i \otimes \bm{y}_j \|\\
\text{cos: } & \overline{ \langle \bm{x}_i,\bm{y}_j \rangle} , \\
\text{sin: } & \overline{ \|\bm{x}_i \otimes \bm{y}_j \|}\\
\nonumber
\end{align*}

\begin{table}
\caption{More detailed ablation study, showing the Hubert NLL 
synthetic dataset $m=5,n=2$.
}
\label{tab:detailed-ablation}
\centering
\begin{tabular}{llllrrrr}
\toprule
 &  &  &  & train & std & test & std \\
method & feature & Node Id & Linear &  &  &  &  \\
\midrule
\multirow[t]{13}{*}{$O(n)$ KAN } & \multirow[t]{2}{*}{all} & \multirow[t]{2}{*}{False} & False & 6.20 & 0.30 & 6.15 & 0.37 \\
 &  &  & True & 6.20 & 0.29 & 6.15 & 0.37 \\
\cline{2-8} \cline{3-8}
 & \multirow[t]{2}{*}{inner-outer} & \multirow[t]{2}{*}{False} & False & 6.20 & 0.30 & 6.16 & 0.37 \\
 &  &  & True & 6.20 & 0.30 & 6.16 & 0.37 \\
\cline{2-8} \cline{3-8}
 & \multirow[t]{2}{*}{n1} & \multirow[t]{2}{*}{False} & False & 6.19 & 0.30 & 6.17 & 0.37 \\
 &  &  & True & 6.18 & 0.30 & 6.17 & 0.37 \\
\cline{2-8} \cline{3-8}
 & \multirow[t]{2}{*}{n12} & \multirow[t]{2}{*}{False} & False & 6.19 & 0.30 & 6.17 & 0.37 \\
 &  &  & True & 6.18 & 0.30 & 6.17 & 0.37 \\
\cline{2-8} \cline{3-8}
 & \multirow[t]{2}{*}{sin-cos} & \multirow[t]{2}{*}{False} & False & 5.90 & 0.31 & 5.91 & 0.41 \\
 &  &  & True & 5.91 & 0.30 & 5.91 & 0.41 \\
\cline{1-8} \cline{2-8} \cline{3-8}
\multirow[t]{13}{*}{$\pi O(n)$ MLP} & \multirow[t]{2}{*}{all} & False & False & 6.23 & 0.41 & 6.22 & 0.53 \\
\cline{3-8}
 &  & True & False & 6.23 & 0.40 & 6.22 & 0.52 \\
\cline{2-8} \cline{3-8}
 & \multirow[t]{2}{*}{inner-outer} & False & False & 6.16 & 0.34 & 6.10 & 0.48 \\
\cline{3-8}
 &  & True & False & 6.19 & 0.30 & 6.17 & 0.37 \\
\cline{2-8} \cline{3-8}
 & \multirow[t]{2}{*}{n1} & False & False & 6.19 & 0.29 & 6.17 & 0.37 \\
\cline{3-8}
 &  & True & False & 6.23 & 0.40 & 6.21 & 0.50 \\
\cline{2-8} \cline{3-8}
 & \multirow[t]{2}{*}{n12} & False & False & 6.14 & 0.36 & 6.15 & 0.40 \\
\cline{3-8}
 &  & True & False & 6.21 & 0.42 & 6.09 & 0.69 \\
\cline{2-8} \cline{3-8}
 & \multirow[t]{2}{*}{sin-cos} & False & False & 5.92 & 0.30 & 5.91 & 0.41 \\
\cline{3-8}
 &  & True & False & 5.92 & 0.30 & 5.91 & 0.41 \\
\cline{1-8} \cline{2-8} \cline{3-8}
\multirow[t]{13}{*}{$O(n)$ MLP} & \multirow[t]{2}{*}{all} & False & False & 6.20 & 0.29 & 6.17 & 0.38 \\
\cline{3-8}
 &  & True & False & 6.20 & 0.29 & 6.18 & 0.38 \\
\cline{2-8} \cline{3-8}
 & \multirow[t]{2}{*}{inner-outer} & False & False & 6.16 & 0.34 & 6.10 & 0.48 \\
\cline{3-8}
 &  & True & False & 6.19 & 0.30 & 6.17 & 0.37 \\
\cline{2-8} \cline{3-8}
 & \multirow[t]{2}{*}{n1} & False & False & 6.19 & 0.29 & 6.17 & 0.37 \\
\cline{3-8}
 &  & True & False & 6.19 & 0.29 & 6.17 & 0.36 \\
\cline{2-8} \cline{3-8}
 & \multirow[t]{2}{*}{n12} & False & False & 6.14 & 0.36 & 6.15 & 0.40 \\
\cline{3-8}
 &  & True & False & 6.18 & 0.30 & 6.09 & 0.49 \\
\cline{2-8} \cline{3-8}
 & \multirow[t]{2}{*}{sin-cos} & False & False & 5.92 & 0.30 & 5.91 & 0.41 \\
\cline{3-8}
 &  & True & False & 5.92 & 0.30 & 5.91 & 0.41 \\
\cline{1-8} \cline{2-8} \cline{3-8}
\bottomrule
\end{tabular}
\end{table}

\begin{figure}[ht]
    \centering
    \includegraphics[width=0.9\linewidth]{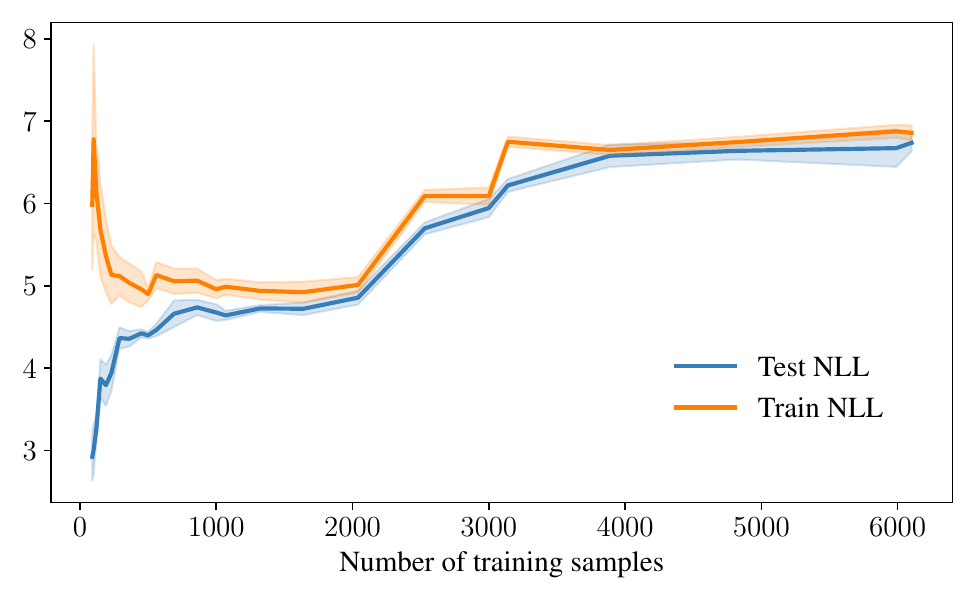}
    \caption{The train and test accuracy in terms of NLL for the Buckyball-catcher system of the MD22 dataset with invariant features inner-outer-n1-n12. }
    \label{fig:sample-size-analysis}
\end{figure}

\begin{table}
\caption{Hubert NLL for the MD22 dataset for $\pi O(n)$ KAN}
\label{tab:sample-size-analysis}
\Centering
\begin{tabular}{crrrr}
\toprule
Num. samples & Train NLL  & Test NLL  \\
\midrule
100 & 6.78 $^{1.15}$ & 3.01 $^{ 0.31}$ \\
500 & 4.90 $^{0.08}$ & 4.40 $^{ 0.04}$ \\
1000 & 4.96 $^{0.11}$ & 4.68 $^{ 0.10}$ \\
3000 & 6.09 $^{0.11}$ & 5.95 $^{ 0.11}$ \\
6102 & 6.86 $^{0.09}$ & 6.74 $^{ 0.10}$ \\
\bottomrule
\end{tabular}
\end{table}

\section{Efficiency of training: group invariance and number of samples}
\label{sec:samples}
In \cref{fig:sample-size-analysis} and \cref{tab:sample-size-analysis}, we show the effect of the sample size for training the invariant representation. In order to obtain a high test accuracy, more than $50\%$ of the data is necessary. 

\section{Additional Related work}

\paragraph{Symmetry preserving machine learning architecture} 
Machine learning interatomic potentials (MLIPs) 
have emerged as powerful tools for modeling interatomic interactions in molecular and materials systems, offering a computationally efficient alternative to traditional ab initio methods. Architectures like Schnet \cite{schutt2017schnet} use continuous-filter convolutional layers to capture local atomic environments and message passing, enabling accurate predictions of molecular properties. To further enhance physical expressivity, $E(3)$-equivariant architectures \cite{thomas2018tensor} have been developed, which respect the symmetries of Euclidean space (rotations, translations, and reflections) by design. These models, such as Tensor Field Networks \cite{thomas2018tensor} and NequIP \cite{batzner2022nequip}, ensure that predictions (i.e. energy and forces) are invariant or equivariant to transformations in 3D space, making them highly data-efficient for tasks like force field prediction in molecular dynamics. 
MACE \cite{Batatia_Kovács_Simm_Ortner_Csányi_2023} is a higher-order equivariant message-passing network that enhances force field accuracy and efficiency by leveraging multi-body interactions. 
E(n)-equivariant GNNs (EGNNs) \cite{Satorras_Hoogeboom_Welling_2022} implement a higher-order representation while maintaining equivariance to rotations, translations, and permutations. 
Irreducible Cartesian Tensor Potential (ICTP)  \cite{Zaverkin_Alesiani_Maruyama_Errica_Christiansen_Takamoto_Weber_Niepert_2024} introduces irreducible Cartesian tensors for equivariant message passing, offering computational advantages over spherical harmonics in the small tensor rank regime. Tensor field networks \cite{Thomas_Smidt_Kearnes_Yang_Li_Kohlhoff_Riley_2018} and Equiformer \cite{Liao_Smidt_2023} use spherical harmonics as bases for tensors. While SO3krates \cite{Frank_Unke_Müller_Chmiela_2024a} combines sparse equivariant representations with transformers to balance accuracy and speed.
Additionally, equivariant Clifford networks \cite{ruheCliffordGroupEquivariant2023b}
extend this framework by incorporating geometric algebra to build equivariant models. 
Equivariant representations mitigate cumulative errors in molecular dynamics  \cite{Unke_Chmiela_Sauceda_Gastegger_Poltavsky_Schütt_Tkatchenko_Müller_2021},  while 
directional message passing with spherical harmonics improves angular dependency modeling as implemented in DimeNet \cite{Gasteiger_Groß_Günnemann_2022}.  
Equivariant or invariant architectures enhance data efficiency, accuracy, and physical consistency in tasks where input symmetries (e.g., rotation, reflection, translation) dictate output invariance or equivariance.
In collider physics, jet-tagging is the problem of identifying the type of particles that have generated the particle collision jet. The collision jet exhibits space-time symmetry, the Lorentz boost. Symmetry-preserving architecture for the Lorentz group have been proposed architecture based on high-order tensor products as LoLa \cite{butter2018deep}, LBN \cite{erdmann2019lorentz} LGN \cite{bogatskiy2020lorentz}, and LorentzNet \cite{gong2022efficient}, which introduce Minkowski dot product attention. Finally, permutation preserving models have been proposed to model function over sets, as DeepSet and subsequent models \cite{zaheer2017deep,amir2023neural}.
While these advancements have significantly improved the accuracy and efficiency of MLIPs for applications in chemistry, physics, and materials science, the advantage of KAN architecture has not yet been explored, we thus take a fundamental step in this direction with our study. 

\paragraph{KAN Architectures}
Kolmogorov-Arnold Networks (KANs) are inspired by the Kolmogorov-Arnold representation theorem, which provides a theoretical foundation for approximating multivariate functions using univariate functions and addition. Early work by Hecht-Nielsen (1987) \cite{hecht1987kolmogorov} introduced one of the first neural network architectures based on this theorem, demonstrating its potential for efficient function approximation. 
\cite{lai2021kolmogorov} study the approximation capability of KST-based models in high dimensions and how they could potentially break the curse of dimension \cite{poggio2022deep}. 
\cite{ferdausKANICEKolmogorovArnoldNetworks2024} propose to combine  Convolutional Neural Networks (CNNs) with Kolmogorov Arnold Network (KAN) principles.
Additionally, 
\cite{yangKolmogorovArnoldTransformer2024}
explored the integration of KAN principles into transformer models, achieving improvements in efficiency for sequence modeling tasks. 
\cite{huEKANEquivariantKolmogorovArnold2024a} propose EKAN, an approximation method for incorporating matrix group equivariance into KANs. While these studies highlight the versatility of KAN architectures in adapting to various neural network frameworks, the extension to physical and geometrical symmetries has not been fully considered.

\paragraph{Application of KAN}
KANs have been applied to a range of machine learning tasks, particularly in scenarios requiring efficient function approximation. For instance, Kůrková (1991) \cite{kuurkova1992kolmogorov} demonstrated the effectiveness of KANs in high-dimensional regression problems, where traditional neural networks often struggle with scalability. In the natural language processing domain, \cite{Galitsky2024} utilized KAN for word-level explanations.  
Furthermore, 
\cite{decarlo2024kolmogorovarnoldgraphneuralnetworks}
applied KANs to graph-based learning tasks, showing that their hybrid models could achieve state-of-the-art results in graph classification and node prediction. 
KAN has been used as a function approximation to solve PDE \cite{wang2024kolmogorovarnoldinformedneural,shukla2024comprehensivefaircomparisonmlp} for both forward and backward problems with highly complex boundary and initial conditions.
\cite{aghaei2024rkanrationalkolmogorovarnoldnetworks} extends KAN with rational polynomials basis to regression and classifications problems. \cite{Seydi2024} explores using Wavelet as basis functions to model hyper-spectral data. KANs have been extended to model time-series \cite{xu2024kolmogorovarnoldnetworkstimeseries,inzirillo2024sigkansignatureweightedkolmogorovarnoldnetworks} to dynamically adapt to temporal data. 
While these, and other \cite{somvanshiSurveyKolmogorovArnoldNetwork2024},  applications highlight the practical utility of KANs in solving complex real-world problems, a significant class of molecular applications remains overlooked. 




\end{document}